%% file: template.tex
\documentclass{article}

\usepackage{arxiv}

\usepackage[utf8]{inputenc} 
\usepackage[T1]{fontenc}    
\usepackage{hyperref}       
\usepackage{url}            
\usepackage{booktabs}       
\usepackage{amsmath,amssymb,amsfonts}
\usepackage{amsfonts}       
\usepackage{nicefrac}       
\usepackage{microtype}      
\usepackage{cleveref}       
\usepackage{lipsum}         
\usepackage{graphicx}
\usepackage{natbib}
\usepackage{doi}

\usepackage{mathtools}
\usepackage{balance}
\usepackage{subcaption}

\newcommand{\inputspace}{\ensuremath{\mathcal{X}}} 
\newcommand{\labelspace}{\ensuremath{\mathcal{Y}}} 

\usepackage{tikz}
\usepackage{scalerel}
\usetikzlibrary{svg.path}
\usepackage{pgfplots}
\pgfplotsset{compat=1.17}
\usepackage{svg}
\usetikzlibrary{positioning,fit,chains,calc}
\def\ConvColor{rgb:yellow,5;red,2.5;white,5}
\def\ReluColor{rgb:yellow,1;red,5;white,6}
\def\AvgColor{rgb:blue,3;red,3;white,8}
\def\AbsColor{rgb:green,3;white,8}
\def\GrayColor{rgb:black,1;white,4}

\DeclareMathOperator*{\argmin}{arg\,min}
\usepackage{amsthm}

\theoremstyle{plain}
\newtheorem{theorem}{Theorem}[section]

\newtheorem{lemma}[theorem]{Lemma}

\theoremstyle{definition}

\theoremstyle{remark}

\usepackage[textsize=tiny]{todonotes}

\newif\ifdraft
\newcommand{\blue}[1]{\ifdraft{\leavevmode\color{blue}{{#1}}}\else{{#1}}\fi}

\title{Quantification using Permutation-Invariant Networks based on Histograms}

\author{ \href{https://orcid.org/0000-0002-4527-6698}{\includegraphics[scale=0.06]{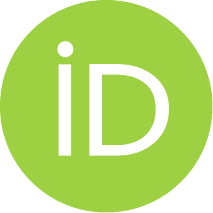}\hspace{1mm}Olaya Pérez-Mon}\\
	Artificial Intelligence Center\\
	University of Oviedo\\
	Gijón, 33204, Asturias, Spain \\
	\texttt{UO257319@uniovi.es} \\
	\And
	\href{https://orcid.org/0000-0002-0377-1025}{\includegraphics[scale=0.06]{orcid.pdf}\hspace{1mm}Alejandro Moreo} \\
	Istituto di Scienza e Tecnologie dell’Informazione\\
	Consiglio Nazionale delle Ricerche\\
	Pisa, 56124, Italy \\
	\texttt{alejandro.moreo@isti.cnr.it} \\
        \And
	\href{https://orcid.org/0000-0002-4288-3839}{\includegraphics[scale=0.06]{orcid.pdf}\hspace{1mm}Juan José del Coz} \\
	Artificial Intelligence Center\\
	University of Oviedo\\
	Gijón, 33204, Asturias, Spain \\
	\texttt{juanjo@uniovi.es} \\
        \And
	\href{https://orcid.org/0000-0002-9250-0920}{\includegraphics[scale=0.06]{orcid.pdf}\hspace{1mm}Pablo González} \\
	Artificial Intelligence Center\\
	University of Oviedo\\
	Gijón, 33204, Asturias, Spain \\
	\texttt{gonzalezgpablo@uniovi.es} \\
}

\hypersetup{
pdftitle={A template for the arxiv style},
pdfsubject={q-bio.NC, q-bio.QM},
pdfauthor={David S.~Hippocampus, Elias D.~Striatum},
pdfkeywords={First keyword, Second keyword, More},
}

\begin{document}
\maketitle

\begin{abstract}Quantification, also known as class prevalence estimation, is the supervised learning task in which a model is trained to predict the prevalence of each class in a given bag of examples. This paper investigates the application of deep neural networks to tasks of quantification in scenarios where it is possible to apply a symmetric supervised approach that eliminates the need for classification as an intermediary step, directly addressing the quantification problem.  
Additionally, it discusses existing permutation-invariant layers designed for set processing and assesses their suitability for quantification. In light of our analysis, we propose HistNetQ, a novel neural architecture that relies on a permutation-invariant representation based on histograms that is specially suited for quantification problems. 
Our experiments carried out in the only quantification competition held to date, show that HistNetQ outperforms other deep neural architectures devised for set processing, as well as the state-of-the-art quantification methods.
Furthermore, HistNetQ offers two significant advantages over traditional quantification methods: i) it does not require the labels of the training examples but only the prevalence values of a collection of training bags, making it applicable to new scenarios; and ii) it is able to optimize any custom quantification-oriented loss function.
\end{abstract}


\section{Introduction}
\label{sec:introduction}

In many real-world applications \citep{Beijbom2015,Forman2006,gonzalez2019automatic,Hopkins2010,Moreo:2022bf,dias2022classification}, predicting the class of each individual example in a dataset is of little concern, since the real interest lies in the \emph{aggregate} level, i.e., in estimating the prevalence of the classes in a bag of examples. Quantification, also known as class prevalence estimation, is the supervised learning task that tackles this particular problem \citep{Gonzalez2017}. 
Quantification has already proven useful in a wide variety of fields, providing answers to questions as for example:
\emph{what is the percentage of positive, neutral, and negative reviews for a specific product of a given company?} \citep{Moreo:2022bf} or \emph{what is the percentage of plankton organisms belonging to each of the phytoplankton species in this water sample?} \citep{gonzalez2019automatic}.

This learning problem can be formalized as follows. Let  $\labelspace=\{c_j\}_{j=1}^l$ be the classes of interest, the goal is to learn a quantifier: $q: \mathbb{N}^\inputspace \rightarrow \Delta^{l-1}$, i.e., a functional $q\in\mathcal{Q}$
that, given a test bag 
$B=\{\mathbf{x}_i\}_{i=1}^m$ in which $\mathbf{x}_i \in \inputspace$ 
is a vector of features representing a data example, 
returns a vector of class prevalence estimations $q(B)\in\Delta^{l-1}$, where $\Delta^{l-1}=\{(p_{1}, \ldots,p_{l}) \mid p_{j}\in[0,1], \sum_{j=1}^{l} p_{j}=1\}$ represents the probability simplex, 
i.e., the domain of all vectors representing probability distributions over $\labelspace$. We will use $\mathbf{p}_B\in \Delta^{l-1}$ to indicate the true prevalence values of a bag 
$B$, and $\hat{\mathbf{p}}_{B}^{A}\in \Delta^{l-1}$ to indicate the estimated prevalence values predicted by the quantification algorithm  $A$, so that $p_{B}(c_j)$ and $\hat{p}_{B}^{A}(c_j)$ are the true and the predicted class prevalence, for class $c_j$, respectively.

At first glance, quantification seems a task very similar to classification in spirit. Indeed, the most straightforward solution to the quantification problem, called Classify \& Count (CC) in the literature, comes down to first learning a hard classifier $h: \inputspace \rightarrow \labelspace$ using a training dataset $D=\{(\mathbf{x}_i,y_i)\}_{i=1}^n$ drawn from $\inputspace \times \labelspace$, to then issue label predictions for all examples in the test bag $B$, and finally counting the number of times each class has been attributed. However, it has been observed that CC gives rise to biased estimators of class prevalence \citep{Forman2008}. The reason is that $h$ is biased towards the training prevalence and therefore tends to underestimate (resp. overestimate) the true prevalence of a class when this class becomes more prevalent (resp. less prevalent) in the test bag $B$ than it was in the training set $D$. 
Noticeably, most quantification algorithms rely on the predictions of a classifier\footnote{
Other alternatives exist which instead rely directly on the features of the examples (the covariates) \citep{GonzalezCastro2013,kawakubo2016computationally}; however, the literature has shown that these approaches tend to be less competitive.} which are subsequently post-processed using information from $D$ and $B$. 
This post-processing is necessary since, in quantification, we assume to face a shift in the data distribution (i.e., that the prevalence of the classes may differ between $D$ and $B$).

This particular shift is generally known as ``label shift" or ``prior probability shift" \citep{quionero2009dataset}, according to which the prior distribution $P(Y)$ 
can change between training and deployment conditions, while the class-conditional densities $P(X|Y)$ 
are assumed stationary. 
The fact that CC is not suitable for quantification under prior probability shift conditions has led to the development of a myriad of methods designed specifically for quantification, which is by now recognized as a task on its own right (see, e.g., \citet{gonzalez2017review, quantbook2023} for an overview).

One of the main advantages of adopting deep neural network architectures (DNNs) for quantification is that DNNs allow the learning process to handle bags of examples (labeled by their class prevalence values) instead of individual examples (labeled by class). 
\blue{Following this intuition, a change in the learning paradigm with respect to the traditional one was first proposed in \citet{Qi2021}.
In this paper, we offer an in-depth exploration of the implications of this change of paradigm, by analyzing the main advantages and limitations with respect to traditional approaches to quantification.
Conversely, traditional quantification methods adopt an \emph{asymmetric} approach in which a classifier is trained to infer the class of the individual examples and in which the label predictions are used to estimate the prevalence of the classes in the bag. This way,} the training labels (\blue{class labels} attached to the example) and the labels to be predicted (\blue{class prevalence values} attached to the bag) are not \emph{homologous}. In contrast, following the approach proposed in \citet{Qi2021}, we can reframe the quantification problem as a \emph{symmetric} supervised learning task in which the training set consists of a collection of bags containing examples labeled at the aggregate level (i.e., without individual class labels).
This formulation posits the quantification problem as a multivariate regression task, in which the labels provided for training and the labels we need to predict become homologous. 
Throughout this paper, we will demonstrate further advantages of this formulation. Among them, and in contrast to traditional quantification methods, the quantifier becomes capable of optimizing any specific loss function. 

With this aim, our paper investigates the application of DNNs to the symmetric quantification problem.
\blue{The paper begins by addressing a central issue that arises when making predictions for entire bags rather than for individual examples, namely, how to represent bags in a permutation-invariant manner}
\citep{edwards2016towards,janossypoolICLR2019,wagstaff2019limitations}.
Two influential DNN architectures have been proposed for set processing: DeepSets \citep{zaheer2017deep} and SetTransformers \citep{lee2019set}. The former employs a pooling layer like max,  average, or median, to summarize each bag,
while the latter uses a transformer architecture without positional encoding. 
These approaches were designed as universal approximation functions for set-based problems. 
Here, we propose a new architecture, called HistNetQ, relying on histogram-based layers. The rationale why histograms seem promising is two-fold: histograms are naturally geared towards representing densities and convey more information than plain statistics (like the mean, or median). We will show that histogram-based layers can be seen as a generalization of the pooling layers proposed in \citet{zaheer2017deep,Qi2021}. 

The contributions of this paper are three-fold. 
First, we analyze the symmetric approach \blue{of \citet{Qi2021}} for quantification, discussing its strengths and limitations.
Secondly, we empirically assess the suitability of previously proposed permutation-invariant layers to the quantification problem. 
Finally, we propose HistNetQ, a new permutation-invariant architecture based on differentiable histograms, specifically useful for quantification tasks.

Our experiments show two main results: i) HistNetQ outperforms not only \blue{traditional quantification methods and previous general-purpose} DNN architectures for set processing but also state-of-the-art \blue{quantification-specific DNN methods \citep{Esuli2018,Qi2021} in the LeQua~\citep{lequa2022} competition, the only competition entirely devoted to quantification held to date}, ii) HistNetQ proves competitive also under the asymmetric approach too, that is, when a set of training bags is not available and must be generated from $D$ via sampling. 


\section{Related Work}
\label{sec:related}

\blue{This section briefly describes the most important quantification methods based on the asymmetric approach as well as DNN architectures specifically designed to handle set-based data. }


\subsection{Quantification Methods}
\label{sec:traditional}

The Adjusted Classify and Count (ACC) method (see~\citet{Vaz:2019eu,Forman2008}), later renamed as 
Black Box Shift Estimation \blue{``hard'' (BBSE-hard) in \citet{lipton2018detecting}}, learns a classifier $h$ and then applies a correction relying on the law of total probability:
%
\begin{align}
  \label{eq:AC} 
  p(h(\mathbf{x})=c_{i}) = \sum_{c_{j}\in
  \labelspace}p(h(\mathbf{x})=c_{i}|c_{j})\cdot p(c_{j}),
\end{align}
which corresponds to the following linear system:
%
\begin{align}
  \label{eq:AC2} 
  \hat{\mathbf{p}}^{\mathrm{CC}}_{B} = \mathbf{M}_h
  \cdot \mathbf{p},
\end{align}
\noindent where $\hat{\mathbf{p}}^\mathrm{CC}_{B}$ are the prevalence estimates returned by the CC method for the test bag $B$ and $\mathbf{M}_h$ is the misclassification matrix characterizing $h$, that is, $m_{ij}$ is the probability that $h$ predicts $c_{i}$ if the true class is $c_{j}$. $\mathbf{M}_h$ is unknown but can be estimated via cross-validation. ACC comes down to solving (\ref{eq:AC2}) as $\hat{\mathbf{p}}^{\mathrm{ACC}}_B = \hat{\mathbf{M}}_h^{-1}\cdot \hat{\mathbf{p}}^{\mathrm{CC}}_{B}$ if $\mathbf{M}_h$ is invertible; otherwise, the Penrose pseudoinverse can be used \citep{Bunse:2022oj}.

%
In \citet{Bella2010}, the authors propose two probabilistic variants of CC and ACC, that consist of replacing the hard classifier $h$ with a soft classifier $s: \inputspace \rightarrow \Delta^{l-1}$, thus giving rise to Probabilistic Classify \& Count (PCC): 
%

\begin{equation}
    \label{eq:pcc}
    \hat{\mathbf{p}}_{B}^{\mathrm{PCC}} = \frac{\sum_{\mathbf{x}\in B}s(\mathbf{x})}{{|B|}},
\end{equation}
and Probabilistic Adjusted Classify and Count (PACC) \blue{(also known as BBSE-soft in \citet{lipton2018detecting})}:   

\begin{equation}
    \label{eq:PAC} 
    \hat{\mathbf{p}}^{\mathrm{PACC}}_{B} = \hat{\mathbf{M}}_s^{-1} \cdot \hat{\mathbf{p}}^{\mathrm{PCC}}_{B}.
  \end{equation}

The Expectation Maximization for Quantification (EMQ) \citep{Saerens2002} method applies the EM algorithm to adjust the posterior probabilities generated by a soft classifier $s$ to the potential shift in the label distribution 
by iterating over a mutually recursive step of expectation (in which the posteriors are updated) and maximization (in which the priors are updated) until convergence. 
The literature has convincingly shown that EMQ is a ``hard to beat'' quantification method \citep{alexandari2020maximum,esuli2020critical}. However, the performance of EMQ heavily relies on the quality of the posterior probabilities generated by $s$ (i.e., on the fact that these posterior probabilities are well-calibrated). \blue{For this reason, different calibration strategies have been proposed in the literature; among these, the Bias-Corrected Temperature Scaling (BCTS) calibration proved the best of the lot \citep{alexandari2020maximum}. In the experiments of Section~\ref{sec:experiments}, we will consider two variants of EMQ: one in which the posterior probabilities are not recalibrated and another in which we apply BCTS.}

The HDy method \citep{GonzalezCastro2013} uses a combination of histograms to represent the distributions of the training data $D$ and the test bag $B$, using the Hellinger Distance (HD) to compare them. HDy builds the histograms using the posterior probabilities returned by a soft classifier $s$. 
Figure~\ref{fig:hdy} illustrates the inner workings of the HDy method. In the training phase, the distributions of the posteriors returned by $s$ for the positive and negative examples in the training set $D$ are estimated using histograms $D^+$ and $D^-$ respectively. At test time, the posteriors of the test bag $B$ are computed and represented using the same procedure. HDy will then return the prevalence value $\hat{p}$ that minimizes the HD between the mixture and the test bag distributions, solving the following optimization problem:
\begin{equation}
\label{eq:hdyeq}
\argmin_{\hat{p} \in [0,1]} HD(  \hat{p}  \cdot D^{+} + (1-\hat{p}) \cdot  D^{-} \ , \ B \ ).    
\end{equation}
\begin{figure}[t]
\centering
\includegraphics[width=0.55\columnwidth]{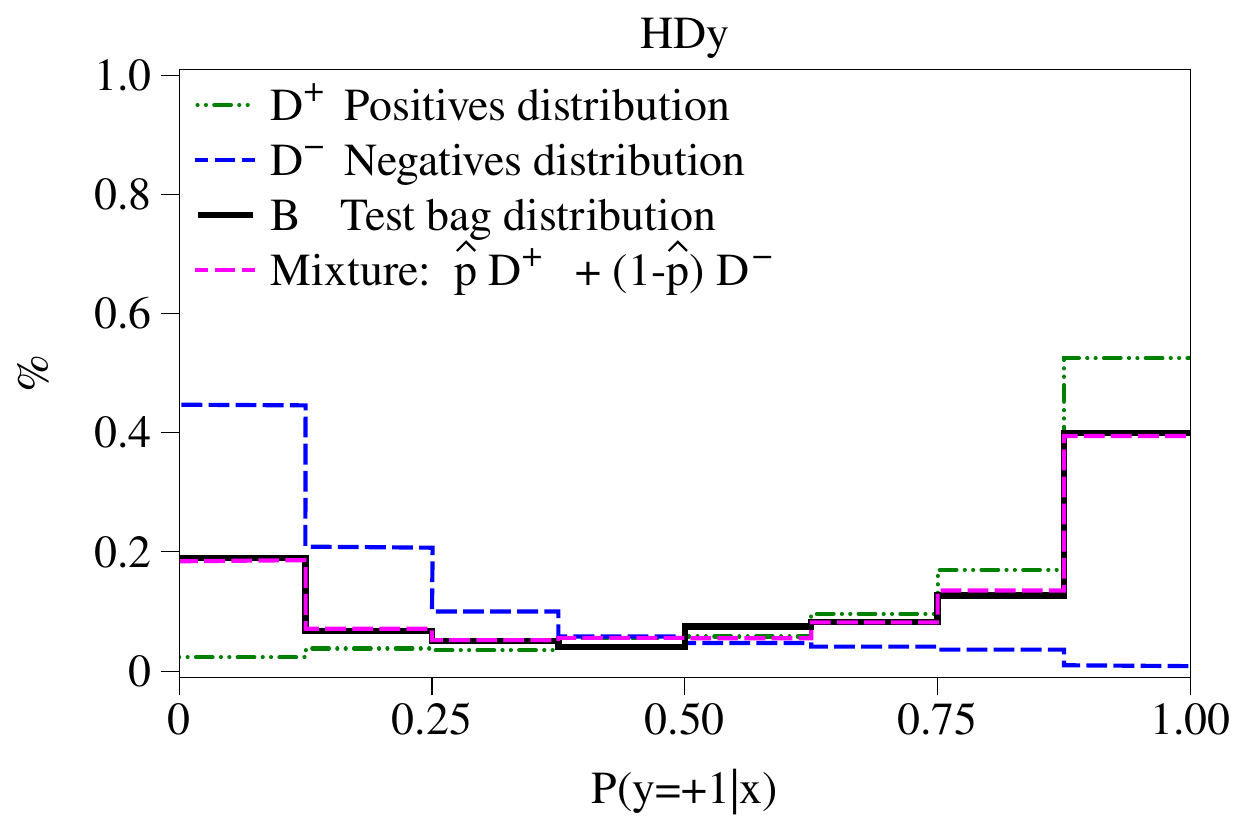}
\caption {In this example, we observe the distributions of positive cases (green) and negative cases (blue) within the training dataset $D$. Additionally, we can see the mixture distribution (magenta) that provides the best approximation of the test bag distribution (black). 
}
\label{fig:hdy}
\end{figure}

QuaNet is a DNN architecture for binary quantification \citep{Esuli2018}.  
QuaNet sorts the inputs by their posterior probabilities and processes the sequence using a bi-directional LSTM that learns a predictor of class prevalence. The prevalence estimation is then combined with the estimates computed with some base quantification methods (CC, ACC, PCC, PACC, and EMQ). QuaNet then generates many bags out of the training data $D$ to train the model. However, in contrast to the rest of the DNN architectures that this paper analyzes, QuaNet follows the asymmetric approach and requires (just like all quantification methods discussed in this section) the availability of a training set $D$ with individual example 
labels.
 
A more exhaustive description of these (and other) quantification algorithms can be found in \citet{quantbook2023,gonzalez2017review}.

\subsection{DNN Architectures for Sets}
\label{sec:DNNs}

In recent years, dedicated DNNs have been proposed to handle set-based data. Even though these architectures were not originally devised with class prevalence estimation in mind, they seem apt for the task since they all construct on top of permutation-invariant representations. 
Quantification requires permutation-invariant layers, because the prevalences of $B$ do not change if the examples in $B$ are shuffled.


The first of these architectures is called DeepSets \citep{zaheer2017deep}. 
DeepSets relies on different permutation-invariant pooling operators, like max, average or median. Pooling operators are applied to the features representing the examples 
in a given bag $B$. An operator is said to be \emph{permutation-invariant} when the output of the layer is not affected by the order in which the examples 
appear in the (serialized) input sequence 
$S$. More formally, a function $f$ is permutation-invariant if $f(S)=f(\pi(S))$ for any permutation function $\pi$. 
\blue{In \citet{Qi2021}, the authors use the same architecture and pooling layers as in DeepSets, proposing its application to quantification problems. For the sake of clarity, we will refer to the use of simple pooling layers, as max, average or median, as DeepSets. }



In \citet{lee2019set} one step forward was taken
by replacing the simple pooling operators of DeepSets with transformers, i.e., with attention-based mechanisms that model complex interactions between the elements in the set. 
In this architecture, called SetTransformers, positional encoding is not included since the order of the examples 
in the bag 
is unimportant. Instead of modeling the interactions between every possible pair of examples, 
SetTransformers incorporates the concept of \emph{inducing points}, learnable latent data points of the vector space 
that is given as input to the self-attention mechanism. 
In this way, 
the original $\mathcal{O}(n^2)$ complexity of SetTransfomer is reduced to $\mathcal{O}(n I)$, where $n$ is the bag 
size and $I$ (with $I\ll n$) the number of inducing points. 
To the best of our knowledge, SetTransformers have never been used in quantification.

\section{Symmetric Quantification: A Case Study Analysis}
\label{sec:shiftparadigm}

\blue{While \citet{Qi2021} pioneered the symmetric approach to the field of quantification learning, the authors did not delve deeper into the implications of the new approach. Among other things, this section aims at filling this gap by providing a comprehensive analysis of its main advantages and limitations.}

Most previous quantification algorithms (as for example those described in Section~\ref{sec:traditional}) require a training dataset $D$, in which labels are attached to individual examples, in order to learn a quantifier $q \in \mathcal{Q}, \ q: \mathbb{N}^\inputspace \rightarrow \Delta^{l-1}$ that, given a test bag $B$, computes estimates of class prevalence. Therefore, the learning device is of the form $L : (\inputspace \times \labelspace)^n \rightarrow \mathcal{Q}$, meaning that the quantification problem is posed as an asymmetric task: training labels are defined in $\labelspace$ while predictions are probability distributions from $\Delta^{l-1}$.
\blue{In order to reformulate quantification as a symmetric supervised task the training set needs to be defined as} $D'=\{(B_i, \mathbf{p}_i)\}_{i=1}^{n'}$, with $B_i\in \mathbb{N}^\inputspace$ a training bag labeled according to its class prevalence values $\mathbf{p}_i \in \Delta^{l-1}$.
The learning device is thus formalized as $L' : (\mathbb{N}^\inputspace \times \Delta^{l-1})^{n'} \rightarrow \mathcal{Q}$,
so that the labels provided for training and the labels we need to predict become homologous, i.e., are both probability distributions in $\Delta^{l-1}$.

This reformulation presents some advantages and disadvantages that were not discussed in \citet{Qi2021}. 
The first advantage \blue{of the new approach is that the quantification method is no longer necessarily bound to prior probability shift. This is a major implication, since most previously proposed methods in the quantification literature assume to be in presence of prior probability shift, and are specifically devised to counter it. In contrast, by adopting the symmetric approach, training examples can potentially exhibit any type of shift, to which the method at hand will try to develop resilience as part of the learning procedure. This characteristic is significant, as it considerably broadens the applicability of the quantification method to scenarios beyond prior probability shift.}

The second advantage is that the quantification problem is addressed \emph{directly}, and not via classification as an intermediate step.
This should be advantageous by virtue of Vapnik's principle, 
according to which \textit{``If you possess a restricted amount of information for solving some problem, try to solve the problem directly and never solve a more general problem as an intermediate step. It is possible that the available information is sufficient for a direct solution but is insufficient for solving a more general intermediate problem''}. Notice that all the methods described in Section~\ref{sec:traditional} (with the sole exception of QuaNet) do not \blue{directly learn a model by minimizing a task-oriented loss (as is rather customary in other areas of supervised machine learning).
The reason is that methods like ACC, PACC, EMQ, and HDy undertake an asymmetric training in which a classifier is learned, and then a predefined post-processing function is employed to yield prevalence estimates. As a result, most quantification methods proposed so far are agnostic to specific quantification loss functions. }
In contrast, methods based on the symmetric approach (including HistNetQ) \blue{can be specifically tailored to minimize a quantification-oriented loss function. This is important as different applications may be characterized by different notions of criticality; well-designed loss functions play a crucial role in accurately reflecting these notions, thereby enabling a method to become accurate in terms of application-dependent requirements. For example (a) one may opt for adopting the absolute error (AE) as an easily interpretable metric in general cases; (b) in applications related to epidemiology, estimating the prevalence of rare diseases might be better served by the relative absolute error (RAE); (c) in different contexts, employing a cost-sensitive error measure could help weigh the relative importance of different classes. See \citep{Sebastiani2020} for a broader discussion on evaluation measures for quantification.}

The third advantage is a widening of the range of problems to which quantification can be applied. Current quantification algorithms can not be applied to problems in which labels are provided at the aggregate level (i.e., datasets of ``type $D'$''). 
Problems in which the supervised training data naturally arise in the form of sets labeled by prevalence are many, and are the object of study of research areas like multi-instance learning \citep{multiinstancereview2010}, and learning from label proportions (LLP) \citep{Freitas:2005qf,Quadrianto:2009lc}. Examples of these problems include, for instance, post-electoral results by census tract,\footnote{See, e.g., the PUMS (public use microdata sample) of the U.S. Census Bureau \url{https://www.census.gov/data/datasets/2000/dec/microdata.html}} 
demographic analysis in which sensible information (e.g., race, gender) is anonymized but provided at the aggregate level, or public records of proportions of diagnosed diseases per ZIP code. Notice that the symmetric approach enables tackling these problems directly.

However, the symmetric approach faces at least two important issues. The first one is that the number of \blue{available training} bags in $D'$ may be limited in some applications. \blue{This limitation arises because supervised learning requires abundant labeled data. While the previous approach requires labeling individuals (i.e., instances), the symmetric approach requires labeling populations (i.e., bags of instances); the latter is certainly more demanding to obtain.} 
\blue{In Section~\ref{sec:sample-mixer} we present a method aimed at mitigating this issue that consists of generating new synthetic bags from existing ones.} 

The second aspect concerns the applicability of the symmetric approach to cases in which \blue{the only available training set is a traditional one, i.e., a dataset of ``type $D$'' with individual class labels. However, note that such a setup poses no real limitation to the symmetric approaches since a} dataset of ``type $D'$'' can be easily obtained from a dataset of ``type $D$'' \blue{via sampling.} Section~\ref{sec:app-protocol} discusses \blue{one sampling generation protocol that fulfill this requirement; the protocol is well-known in the quantification literature although it is more commonly employed for evaluation purposes, i.e., for generating, out of a collection of labelled individuals, many test bags exhibiting different class distributions that are used for testing quantification algorithms}. 
\blue{Of course, while feasible in principle, it remains to be seen whether a symmetric approach trained via a sampling protocol performs comparably in terms of quantification accuracy with respect to a traditional asymmetric approach trained on the original dataset.}
This aspect will be analyzed in the experiments of Section~\ref{sec:experiments}.

\begin{figure}[t!]
\centering

\begin{minipage}{.01\linewidth}
\center \scriptsize \mbox{}
\end{minipage}
\begin{minipage}{.37\linewidth}
\center \scriptsize AE
\end{minipage}
\begin{minipage}{.37\linewidth}
\center \scriptsize RAE
\end{minipage}

\begin{minipage}{.01\linewidth}
\center \scriptsize \rotatebox{90}{T1A}
\end{minipage}
\begin{minipage}{.37\linewidth}
\center \includegraphics[width=\linewidth]{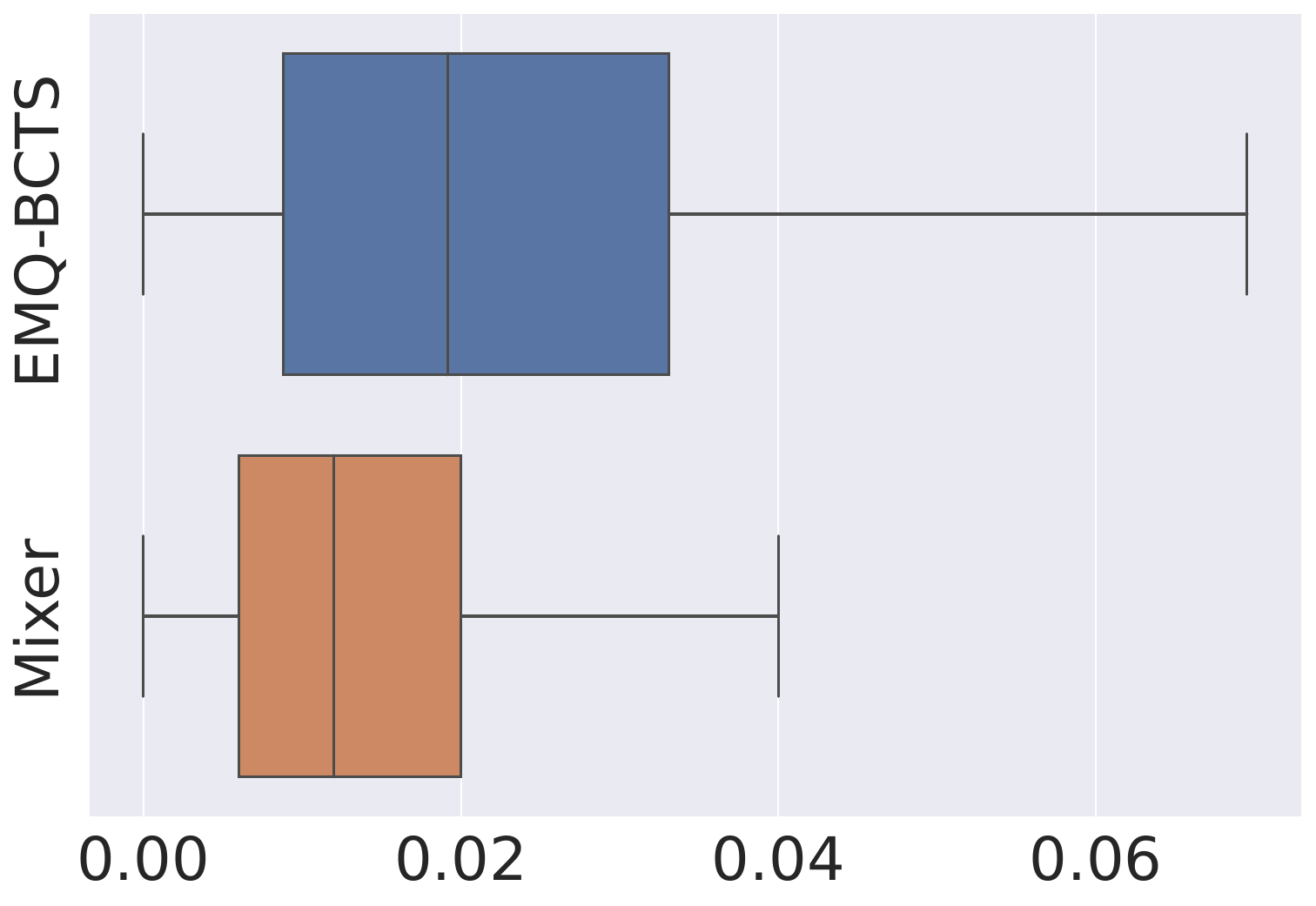}
\end{minipage}
\begin{minipage}{.37\linewidth}
\center \includegraphics[width=\linewidth]{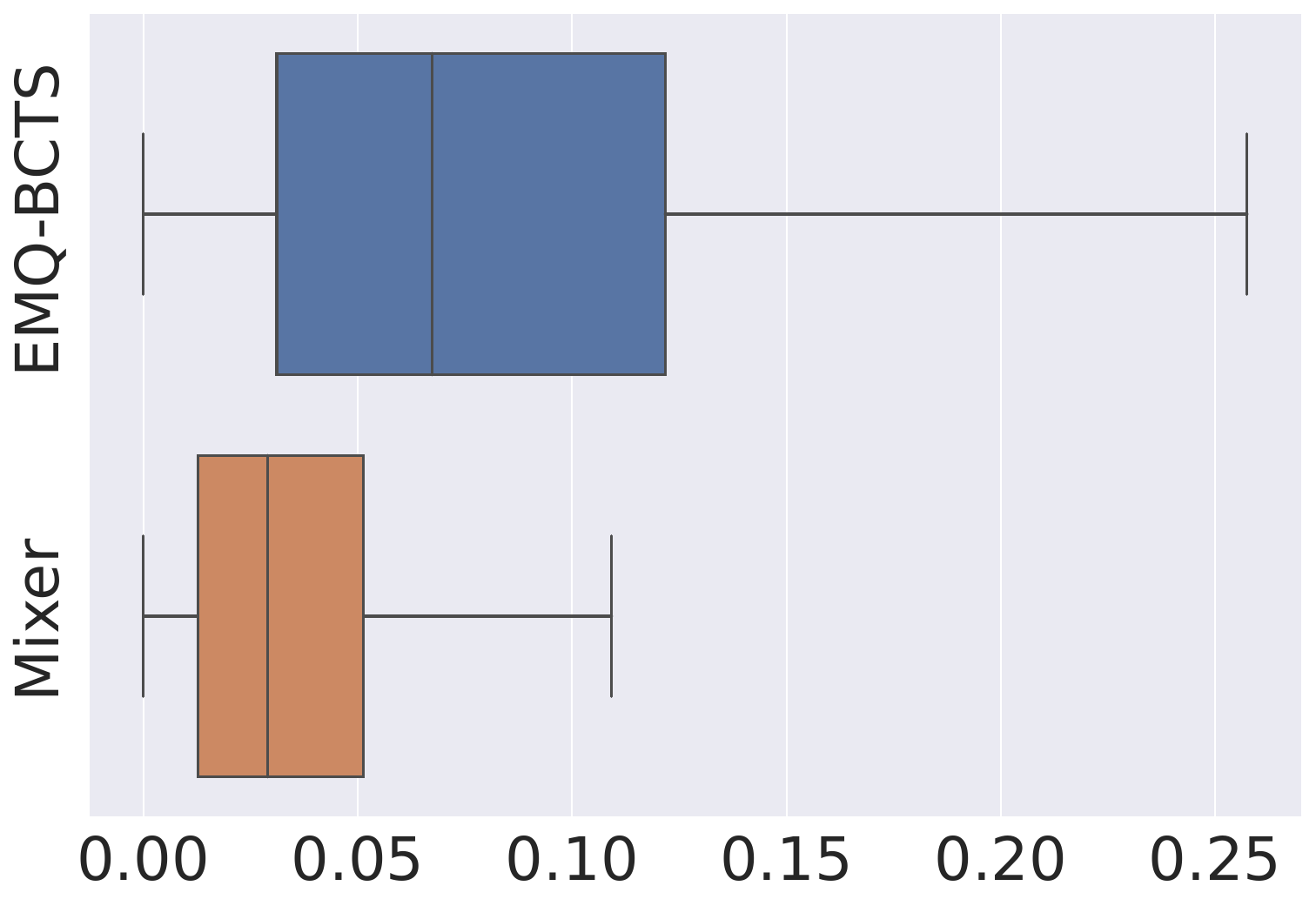}
\end{minipage}

\begin{minipage}{.01\linewidth}
\center \scriptsize \rotatebox{90}{T1B}
\end{minipage}
\begin{minipage}{.37\linewidth}
\center \includegraphics[width=\linewidth]{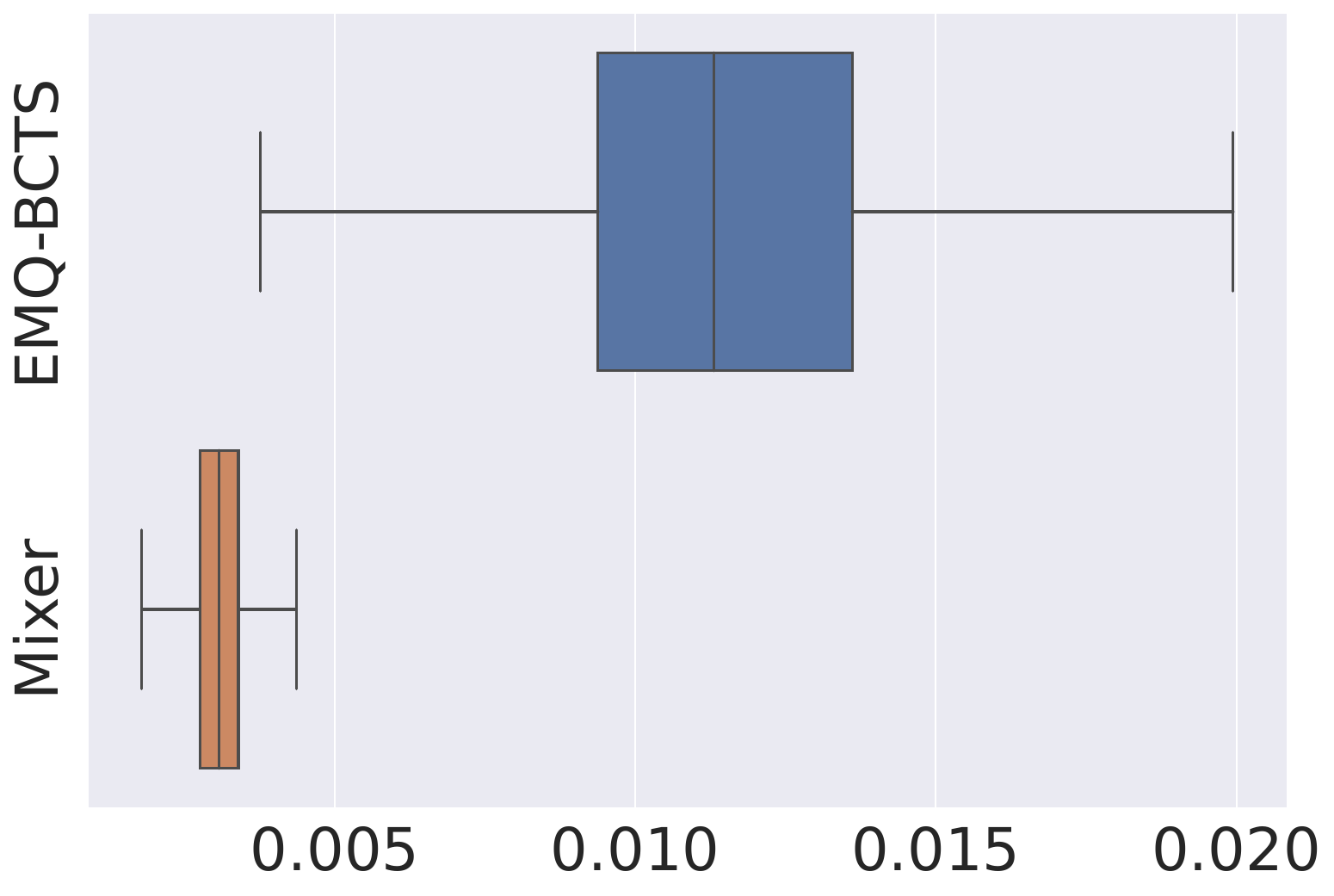}
\end{minipage}
\begin{minipage}{.37\linewidth}
\center \includegraphics[width=\linewidth]{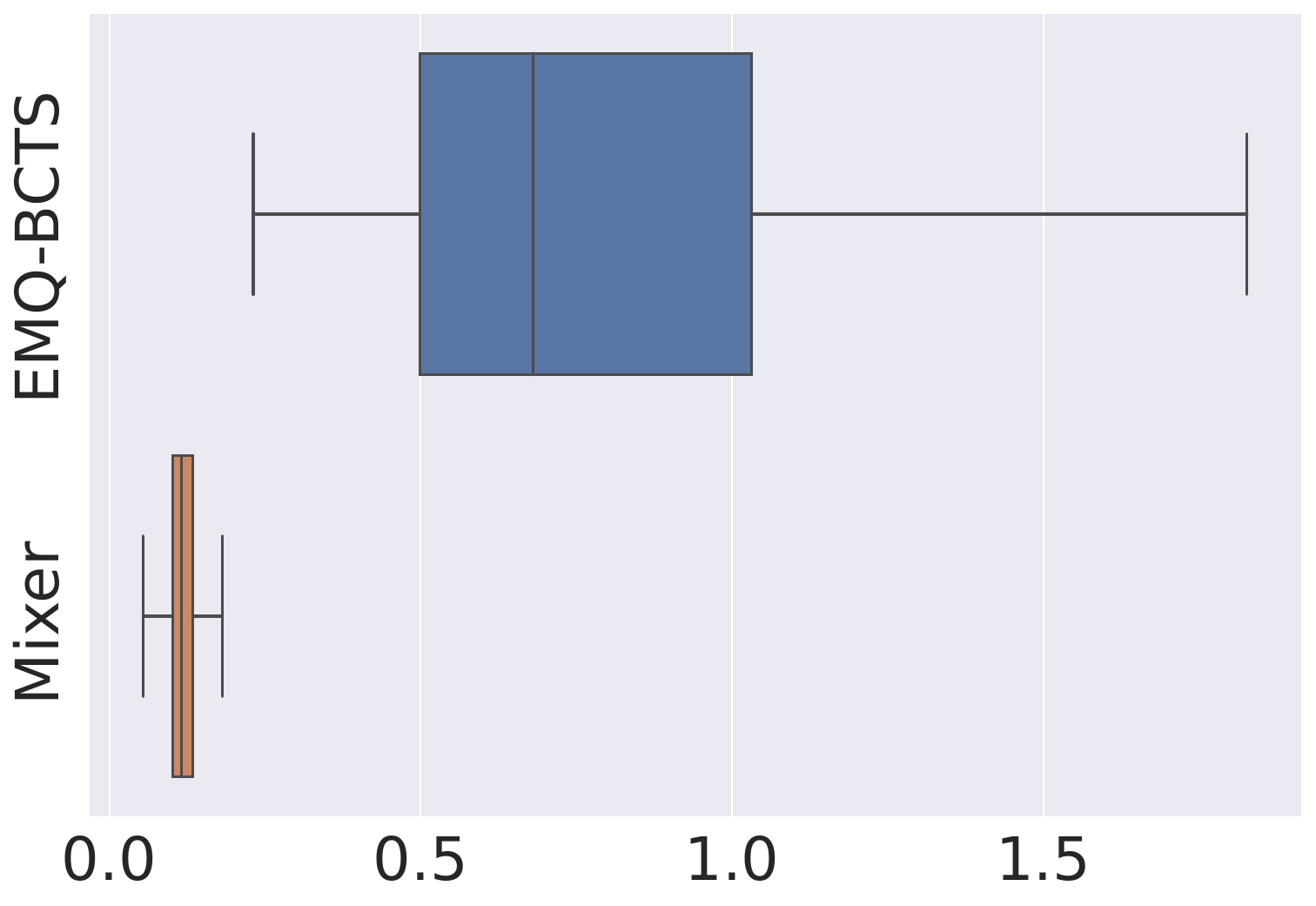}
\end{minipage}

\caption{
Distribution of errors produced by EMQ-BCTS and ``Mixer'' heuristic in terms of Absolute Error (AE) and Relative AE (RAE) as evaluated in LeQua datasets T1A (top row) and T1B (bottom row) (see more details in Section~\ref{sec:experiments}). 
EMQ-BCTS was trained and optimized using, respectively, the training and validation sets, and evaluated in the corresponding test bags, 
while for Mixer we run Montecarlo simulations generating bags 
out of the training examples 
of each task.}
\label{fig:errormixer}
\vskip 0.1in
\end{figure}

\subsection{Bag Mixer: data augmentation for quantification}
\label{sec:sample-mixer}


\blue{Arguably} the most important issue the symmetric approach has to face concerns the potential limited size of $D'$, something that might easily lead to overfitting, especially if DNN methods are used. The reason why is that training bags are the equivalent counterparts of training examples from a classification problem. This means that even a relatively high number of training bags (e.g., the LeQua datasets we use in the experiments of Section~\ref{sec:experiments} comprise 1000 bags each) remains quite low when compared to classification datasets customarily used in deep learning \blue{(that typically comprise tens or hundreds of thousands of instances)}.

One possible solution to this problem comes down to generating new bags out of the original ones via subsampling and mixing. Of course, while we are able to generate new bags out of the examples 
in our dataset, we do not know the (gold) true prevalence of the newly generated bags. 
However, we can guess it and label our new bags 
with (silver) prevalence values instead. The heuristic we propose is called ``Mixer'' and works as follows: given a dataset of type $D'$, at each epoch we generate new training bags 
$(B, \hat{\mathbf{p}})$, from the original ones, in which $B=B'_i\bigcup B'_j$, where $B'_i$ (resp. $B'_j$) is a random subset containing half of the elements of $B_i$ (resp. $B_j$) and $\hat{\mathbf{p}}=(|B'_i|\mathbf{p}_i + |B'_j|\mathbf{p}_j)/(|B'_i|+|B'_j|)$,  
and in which bags 
$B_i$ and $B_j$ are chosen randomly from our original dataset $D'$. Bags generated with the Mixer are fed into the network along with real bags from $D'$. The proportion of real bags used for each iteration is controlled by one hyperparameter.

This heuristic certainly introduces some noise in the labels of the newly generated bags. 
However, we have verified that it is typically much smaller than the error that other surrogate quantifiers would produce if employed in place of the heuristic for estimating the bag 
prevalence (see the experiment in Figure~\ref{fig:errormixer}). 
We use the Bag Mixer for training all DNN methods.

\subsection{Generating a collection of bags from $D$}
\label{sec:app-protocol}

Many experiments in quantification papers use benchmark datasets borrowed from
classification problems.
In these datasets, testing bags are not naturally provided so we generate them artificially
for testing quantification algorithms.
This way, a sampling protocol is employed to generate a sufficiently large collection of testing bags,
$D'\!=\!\{(B_i, \mathbf{p}_i)\}_{i=1}^{m'}$, with $B_i\in \mathbb{N}^\inputspace$ and $\mathbf{p}_i \in \Delta^{l-1}$, from a labeled classification test set 
$T=\{(\mathbf{x}_i,y_i)\}_{i=1}^m$. The most widely adopted sampling protocol is called Artificial-Prevalence Protocol (APP) \citep{Forman2005} which is designed to simulate prior probability shift. APP consists of drawing a fixed number of bags 
in which the bag 
prevalence $\mathbf{p}_i$ is uniformly drawn at random from the probability simplex $\Delta^{l-1}$, and the testing bag 
$B_i$ for each class prevalence $\mathbf{p}_i$ is generated from $T$ via random sampling with replacement, trying to maintain $P(X|Y)$ constant. In order to draw prevalence vectors uniformly at random, we use the Kraemer algorithm \citep{smith2004sampling}

\blue{Note that the APP protocol is also useful for generating a training dataset of ``type $D'$'' from a training dataset of ``type $D$''; that is, when there are no dedicated training bags available but we want to train DNN-based methods using the symmetric approach.}


\blue{While APP is specialized in generating prior probability shift,} notice that if we have some prior knowledge about the application at hand, other sampling protocols designed for reproducing the expected shift could be applied in place \citep{zhang2013domain}.

\section{HistNetQ: Differentiable Histograms}
\label{sec:histnetq}

In this paper, we propose a permutation-invariant layer for quantification that gains inspiration from histograms.
Histograms represent powerful tools for describing sets of values: they are directly aligned with the concept of counting, and they disregard the order in which the values are presented.
However, histograms are not differentiable operators and hence cannot be directly employed as building blocks in a deep learning model.
In order to overcome this impediment, histograms can be approximated by using common differentiable operations such as convolutions and pooling layers.
Different realizations of this intuition have been reported in the literature of computer vision \citep{avi2020deephist,peeples2021histogram,wang2016learnable} but, to the best of our knowledge, no one before has investigated differentiable histograms in quantification.

Previous attempts for devising differentiable histograms differ in how these are implemented. On the one hand,  
\citet{wang2016learnable, peeples2021histogram} proposed soft variants 
in which every value can potentially contribute to more than one bin, based on the distance of the value to the center of the bin and the width thereof.
On the other hand,  in \citet{yusuf2020differentiable} the authors propose a hard variant, that is, every value only contributes to the bin in which the value falls.
Throughout preliminary experiments we carried out using all variants, we found that the differences in performance were rather small.
The hard variant proved slightly better in such experiments (in terms of validation loss) and is our variant of choice for the experiments of Section~\ref{sec:experiments}. Other architectures and their results are discussed in the supplementary material.

More formally, given a bag 
of $n$ data examples 
$B=\{\mathbf{x}_i\}_{i=1}^{n}$, with $\mathbf{x}_i\in\inputspace$, our goal is to compute a histogram for every feature vector $\{\mathbf{f}_k\}_{k=1}^z$, where $\mathbf{f}_k\in\mathbb{R}^n$ represents the values of the $k$-th feature across the $n$ instances in the bag $B$,
and where $z$ is the number of features extracted \blue{(i.e., every histogram is computed along a different column from a $n\times z$ matrix representing $B$)}.
The hard differentiable histogram layer proposed 
in \citet{yusuf2020differentiable}
takes a user-defined hyperparameter $N$ determining the (fixed) number of bins (we use the same number of bins for all feature vectors), and defines $\{(\mu_b^{(k)}, w_b^{(k)})\}_{b=1}^{N}$, the bin centers and widths, as independent learnable parameters for each feature vector $\mathbf{f}_k$. 
The value in the $b$-th bin of the $k$-th histogram is defined by:

\begin{equation}
    H_b^{(k)}(B)=\frac{1}{n}\sum_{i=1}^n \phi(\mathbf{f}_k [i]; \mu_b^{(k)},w_b^{(k)}),
\end{equation}
where $\phi$ is defined by:
\begin{equation}
    \label{eq:binning}
    \phi(v;\mu,w)=
    \begin{cases}
     0, & \text{if}\; 1.01^{w-|v-\mu|}\leq 1\\
     1, & \text{otherwise.}
    \end{cases}
\end{equation}
The value $1.01$ in Equation~\ref{eq:binning} is \blue{justified in \citet{yusuf2020differentiable}} simply \blue{as} a value that yields slightly smaller values than $1$ when the exponent is $<0$ and slightly bigger values than $1$ if the exponent is $>0$.
This, in combination with a threshold operation, results in a (differentiable) mechanism to detect which values fall into which bin (see Figure~\ref{fig:hist} for a graphical representation of the layer).

Note that we compute densities (by dividing the counts by $n$) and not plain counts, in order to factor out the effect of the bag 
size in the final representation. Note also that the total number of parameters of a differentiable histogram layer is $2 N z$.
Since the bin centers and widths are learnable, the output can contain interval ``gaps'' (i.e., intervals in which values are not taken into account), interval overlaps \blue{(thus allowing one value to contribute to more than one overlapping bin at the same time)}, or even zero-width bins.
This means that the output of the layer is not strictly a histogram, but this allows the model to control the complexity of the representation (should $N$ be too high, the model can well learn to overlap bins or create zero-width ones).

It is worth noting that the quantification method HDy, described in Section~\ref{sec:traditional}, \blue{also relies on histograms.} 
However, there are significant differences between HDy and HistNetQ. To begin with, \blue{HistNetQ models histograms on the latent representations of the (potentially high-dimensional) data, whereas HDy models histograms on the posterior probabilities returned by a soft classifier.}
Also, as HistNetQ uses a symmetric approach and learns directly from bags, it does not need to impose any learning assumption, while HDy instead relies the prior probability shift assumptions. Lastly, HistNetQ enables the optimization of a specific loss function during the learning process, \blue{while this is not possible in HDy.}

\begin{figure}[!tbp]
\centering
\begin{tikzpicture}[
    node distance=5mm and 4 mm,
      start chain = going right,
every node/.style = {
    draw,
    minimum height=4em,
    text width=3.5em,
    align=center,
    join},
every join/.style = {->}
            ]
        \node[on chain,fill=\GrayColor] (featuremaps){\scriptsize Feature\\maps};
        \node[on chain,text width=3em]  (conv1)[fill=\ConvColor,below=of featuremaps]{\scriptsize Conv. I\\[0.5em] \tiny \textbf{weights:} fixed 1\\[0.5em]\textbf{bias: }$-\mu$\\};
        \node[on chain,fill=\AbsColor,text width=1.6em]  {\scriptsize Abs};
        \node[on chain,text width=3em]  (conv2)[fill=\ConvColor]{\scriptsize Conv. II\\[0.5em] \tiny \textbf{weights:} fixed -1\\[0.5em]\textbf{bias: }$w$\\};
        \node  [on chain,fill=\ReluColor,text width=1.6em](exp){\scriptsize Exp};
        \node  [on chain,fill=\ReluColor,text width=2.5em]{\scriptsize Thres-\\hold};
        \node  [on chain,fill=\AvgColor,text width=2.5em](avg){\scriptsize Global avg. pooling};
        \node  [on chain,above=of avg,fill=\GrayColor,right=5.85cm of featuremaps](histograms){\scriptsize Features Histograms};
         \draw[red,dotted,thick] ($(conv1.north west)+(-0.1,0.2)$)  rectangle  ($(avg.south east)+(0.1,-0.2)$);
         \coordinate (CENTER) at ($(conv1)!0.5!(avg)$);
        \node  [draw=none,align=center,text width=20em,shift={(0,0.5)},above =of CENTER,]{};
    \end{tikzpicture}
    \caption{Learnable histogram layer with hard binning and learnable 
    bin centers and widths. The individual components are common operations used in DL frameworks that we use to compute Equation~\ref{eq:binning}.}
    \label{fig:hist}
\end{figure}
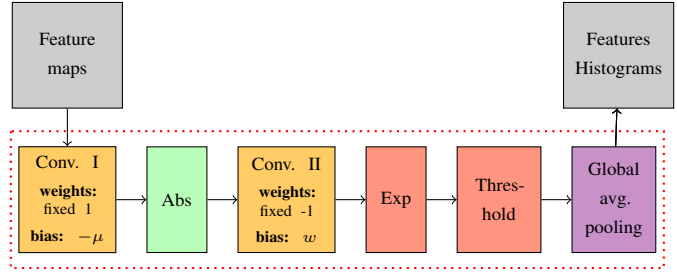
\vspace{1em}
\begin{lemma} Hard differentiable histogram layers are permutation-invariant.
\label{lemma:perminvariant}
\end{lemma}
\begin{proof}
The proof is straightforward. The value $H_b^{(k)}(B)$ 
is computed by summing over the values returned by the $\phi$ function. 
Although $\pi(B)$, with $\pi$ any permutation function, alters the order of the values within the feature vectors $\mathbf{f}_k$,
this ordering does not affect the final counts 
since: 
$$\frac{1}{n}
\sum_{i=1}^n 
\phi(\mathbf{f}_k[i];
\mu_b^{(k)}, w_b^{(k)}) = \frac{1}{n} \sum_{i=1}^n \phi(\pi(\mathbf{f}_k)[i]; \mu_b^{(k)}, w_b^{(k)}),$$ and hence $H_b^{(k)}(B)=H_b^{(k)}(\pi(B))$.
\end{proof}






One of the claims of the paper is that polling layers like average, median, or max proposed for set operations \citep{zaheer2017deep,Qi2021} can be seen as simplified models (or ablations) of our proposal of using histogram layers \blue{(in other words, that a histogram subsumes the information conveyed by these statistics)}. In order to verify this, we designed a toy experiment where a small neural network is trained \blue{to learn each of the aggregation functions (average, median, and max). To this aim, we equip our network with a single histogram layer of 64 bins,} 
followed by just two fully connected layers (sizes 32 and 16). The network is then trained on randomly generated vectors of 100 real values between $[0,max]$, where $max$ is a random number in the range $[0,1]$.
The absolute errors are pretty low: $0.0055$ (average), $0.0090$ (median), and $0.0219$ (max) suggesting that histograms are richer representations than the average, median, or max.
%
%
As the histogram layer can capture the distribution of the data, it provides a more comprehensive view of the data beyond single summary statistics, something that makes them a promising approach for machine learning tasks that require a density estimation method over sets. 

\section{Experiments}
\label{sec:experiments}

We have performed two main experiments.\footnote{The source code for reproducing the experiments is available at \url{https://github.com/a2032/a2032}}
The most important one was based on the datasets\footnote{\url{https://zenodo.org/record/5734465}} provided for the LeQua 2022 quantification competition \citep{lequa2022}. These datasets permitted us to make a perfect comparison between asymmetric and symmetric methods. The LeQua competition consists of four subtasks of product reviews quantification: two subtasks (T2A and T2B) having to do with raw text documents, and two subtasks (T1A and T1B) in which documents were already converted into numerical vectors ($\inputspace \subset \mathbb{R}^{300}$) by the organizers. We focused on T1A and T1B subtasks since we are unconcerned with textual feature extraction in this paper. 
T1A is a binary task of estimating the prevalence of positive versus negative opinions. The organizers provided a training dataset $D$ with 5,000 labeled opinions, a validation set $D'$ with 1,000 bags 
of 250 unlabeled opinions annotated by prevalence, and 5,000 testing bags 
of 250 opinions each. T1B is a multiclass task of estimating the prevalence of 28 merchandise product categories, and consists of a training set $D$  with 20,000 labeled opinions, a validation set $D'$ with 1,000 bags of 1,000 unlabeled documents annotated by prevalence, and 5,000 testing bags of 1,000 documents.

We trained our DNN methods using the validation bags $D'$, in line with the symmetric approach (Section~\ref{sec:shiftparadigm}), while traditional quantification methods (Section~\ref{sec:traditional}) were trained using the training set $D$. Notice that, \blue{as could be expected in most applicative domains, } the size of the latter \blue{(i.e., the number of labelled instances)} is larger than the size of the former \blue{(i.e., the number of labelled bags). In order to compensate this shortage of training bags, we employ the Bag Mixer (Section~\ref{sec:sample-mixer}) to train all DNN methods.} 

The target loss function of the LeQua competition was the relative absolute error:
\begin{equation}
    \label{eq:rae}
    RAE(\mathbf{p},\hat{\mathbf{p}})=\frac{1}{|\labelspace|}\sum_{c_i\in\labelspace}\frac{|\delta(p(c_i))-\delta(\hat{p}(c_i))|}{\delta(p(c_i))}, 
\end{equation}
in which $\delta(p_i) = \frac{p_{i}+\epsilon}{|\labelspace|\epsilon+1}$ \blue{is the smoothing function}, with $\epsilon$ \blue{the smoothing factor that we set to $(2|B|)^{-1}$ following \citet{Forman2008}, where $|B|$ corresponds to the number of instances in the bag $B$}. This section reports relative errors but also absolute errors, $AE(\mathbf{p},\hat{\mathbf{p}})=\frac{1}{|\labelspace|}\sum_{c_i\in\labelspace}|p(c_i)-\hat{p}(c_i)|$, because both have been found to be better suited for quantification evaluation than other measures (like, e.g., KLD), according to \citet{Sebastiani2020}.  
\blue{We have optimized all DNN methods to minimize the RAE loss, because this was the official evaluation measure}. 
\blue{As recalled from Section~\ref{sec:related}, most traditional quantification methods rely on the predictions of an underlying classifier. We use Logistic Regression in all cases.}
The hyperparameters of the classifier were optimized, \blue{independently for each quantification method, in terms of } RAE in the validation bags,
either by the LeQua organizers\footnote{\url{https://github.com/HLT-ISTI/QuaPy/tree/lequa2022}} (CC, PCC, ACC, PACC, HDy, QuaNet) or by ourselves (EMQ-BCTS, EMQ-NoCalib), using the QuaPy quantification library \citep{moreo2021quapy}. 

\begin{figure*}
  \centering
  \includegraphics[trim=0 20 0 20,clip,width=14cm]{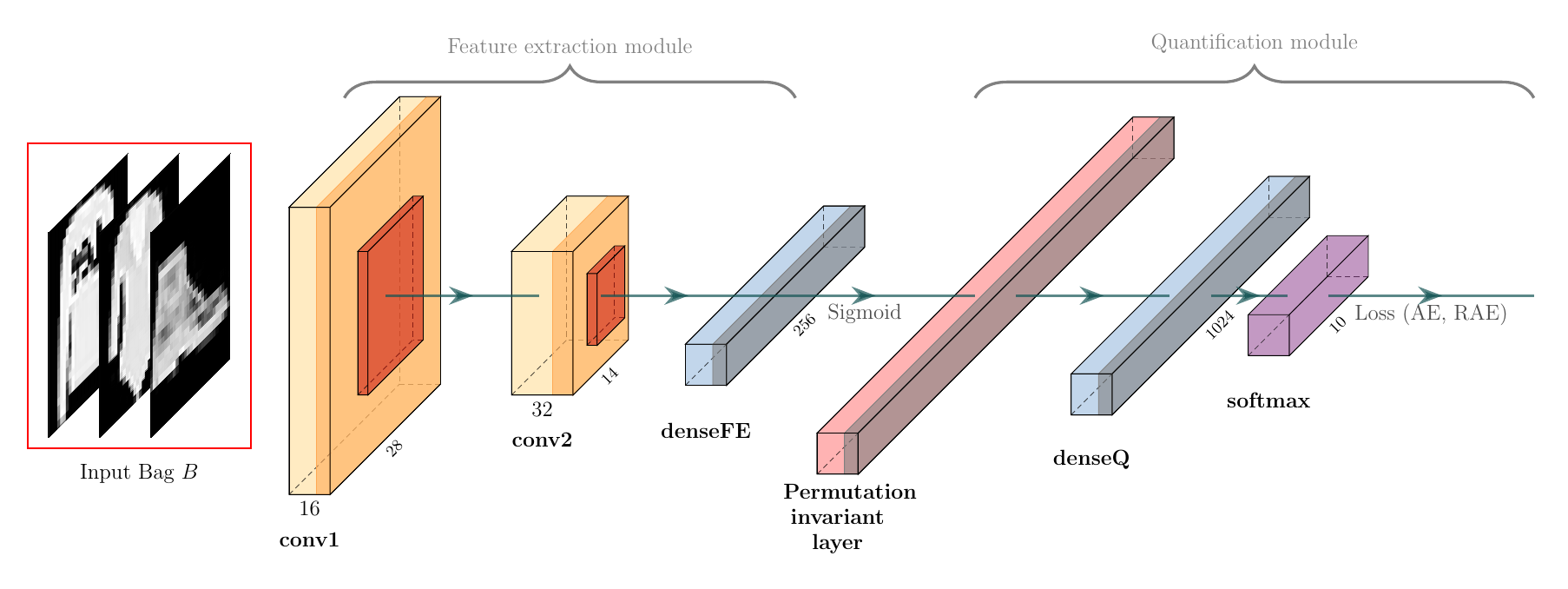}
  \caption{An example of the common architecture used for DNNs methods. The feature extraction layer and the layer sizes correspond to a computer vision problem (Fashion-MNIST dataset). DenseFE and denseQ are sequences of fully-connected layers used in the feature extraction module and in the quantification module, respectively.}
  \label{fig:architecture}
\end{figure*}
We also use the Fashion-MNIST dataset \citep{xiao2017fashionmnist} (a more challenging variant of the well-known MNIST) for the second experiment. In this case, the goal was to analyze the performance of symmetric approaches when the training data consists of individual labeled examples. 
The training set $D$ consists of 60,000 images while the test set consists of 10,000 images of 28x28 pixels. 
Both sets are labeled according to 10 classes. We use the APP protocol (Section~\ref{sec:app-protocol}) for: i) generating the training bags $D'$
for DNNs methods (500 bags 
with 500 examples 
for each epoch), and ii) evaluating all methods (5,000 test bags, 
of 500 examples 
each). To ensure the fairness of the experiment, all the methods used the same feature extraction module (described in Section~\ref{exp:architecture}). The quantifiers learned by DNN methods were optimized using AE or RAE depending on the loss function used. The classifier employed with quantification methods consists of a classification head on top of the feature extraction module, with a softmax activation function, optimized to minimize the cross-entropy loss. We applied early stopping on a validation set in order to prevent overfitting. The validation set was then used to generate the posterior probabilities on which some methods (ACC, PACC) estimate the misclassification rates, and in which EMQ-BCTS optimizes the calibration function.

\begin{table*}[th]
 \caption{Results for \textsc{LeQua-T1A}, \textsc{LeQua-T1B} and \textsc{Fashion-MNIST},  in terms of AE and RAE. 
 Methods that are not statistically significantly different from the best one (\textbf{bold}), according to a Wilcoxon signed-rank test, are marked with $\dag$ if $0.001 \le p\operatorname{-value} \le 0.05$ and with $\ddag$ if $p\operatorname{-value} > 0.05$. Missing values correspond to binary quantifiers in multiclass problems.}
 \label{tab:resultsall}
 \vskip 0.15in
 \centering
 \resizebox{\textwidth}{!}{%
 \centering { \input{tables/table_all_decimals.tex}}}
\end{table*}



\begin{figure*}[t]
 \centering
 \begin{subfigure}[T]{\textwidth}
  \centering
  \includegraphics[trim={0 3.9cm 0 0},clip,width=\textwidth]{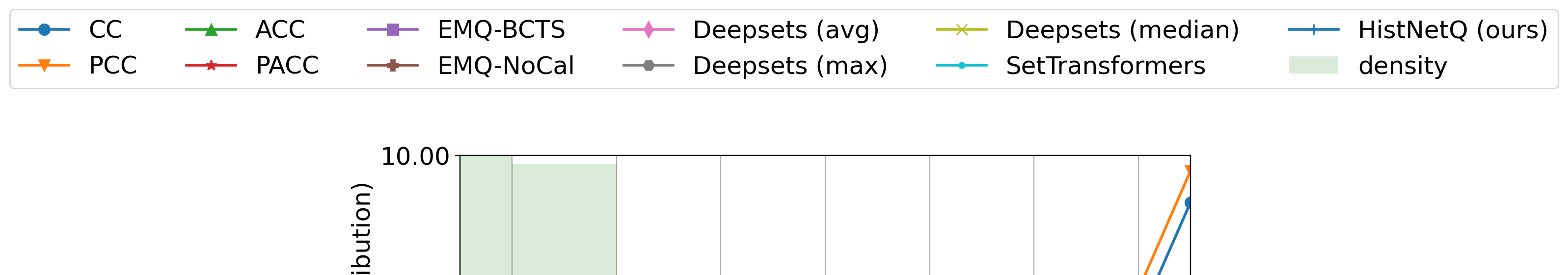}
 \end{subfigure}
 \begin{subfigure}[T]{0.33\textwidth}
  \centering
  \includegraphics[width=5.4cm]{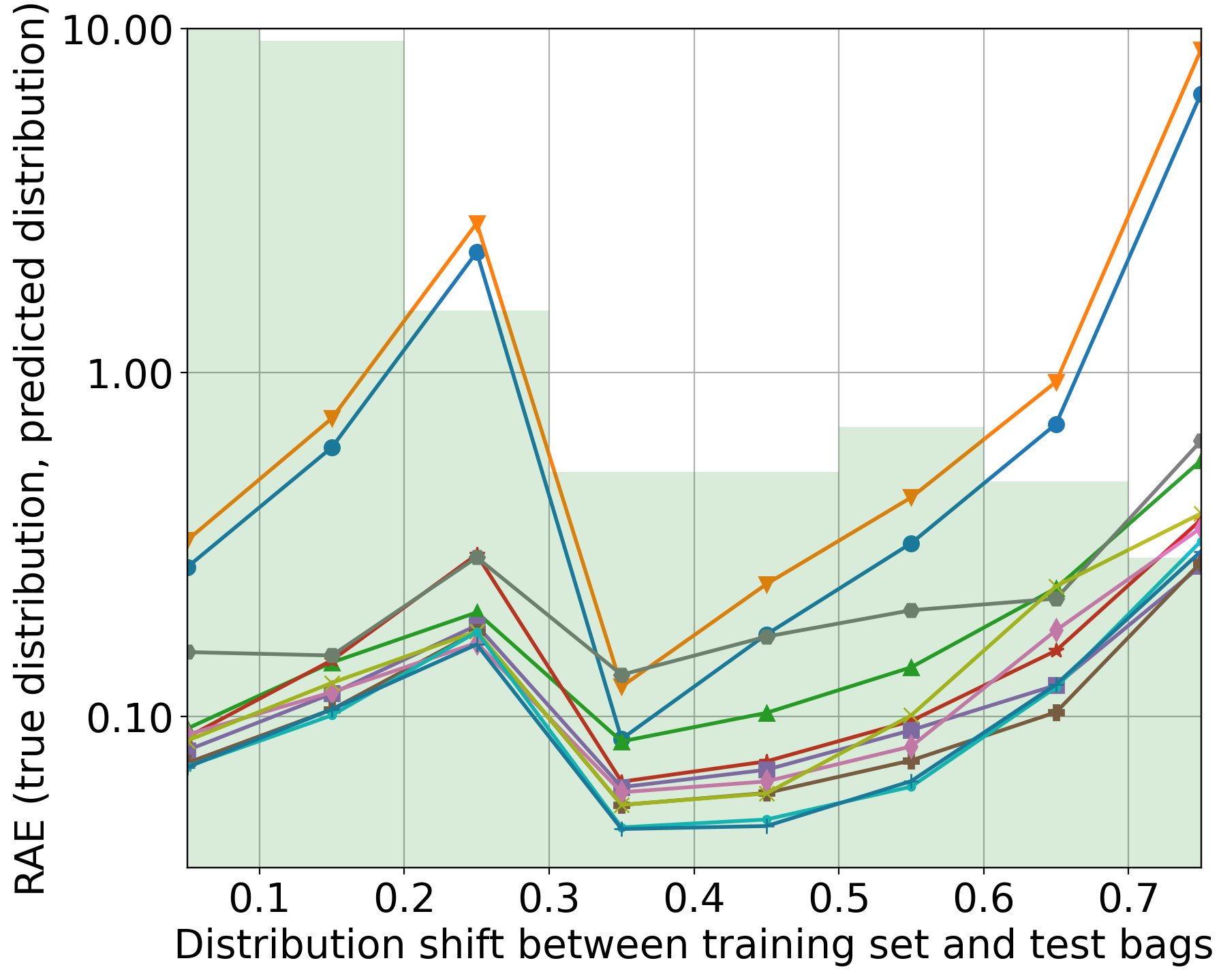}
  \caption{LeQua-T1A}
  \label{fig:t1a}
 \end{subfigure}
 \begin{subfigure}[T]{0.33\textwidth}
  \centering
  \includegraphics[trim = {0.9cm 0 0 0}, clip, width=5.4cm]{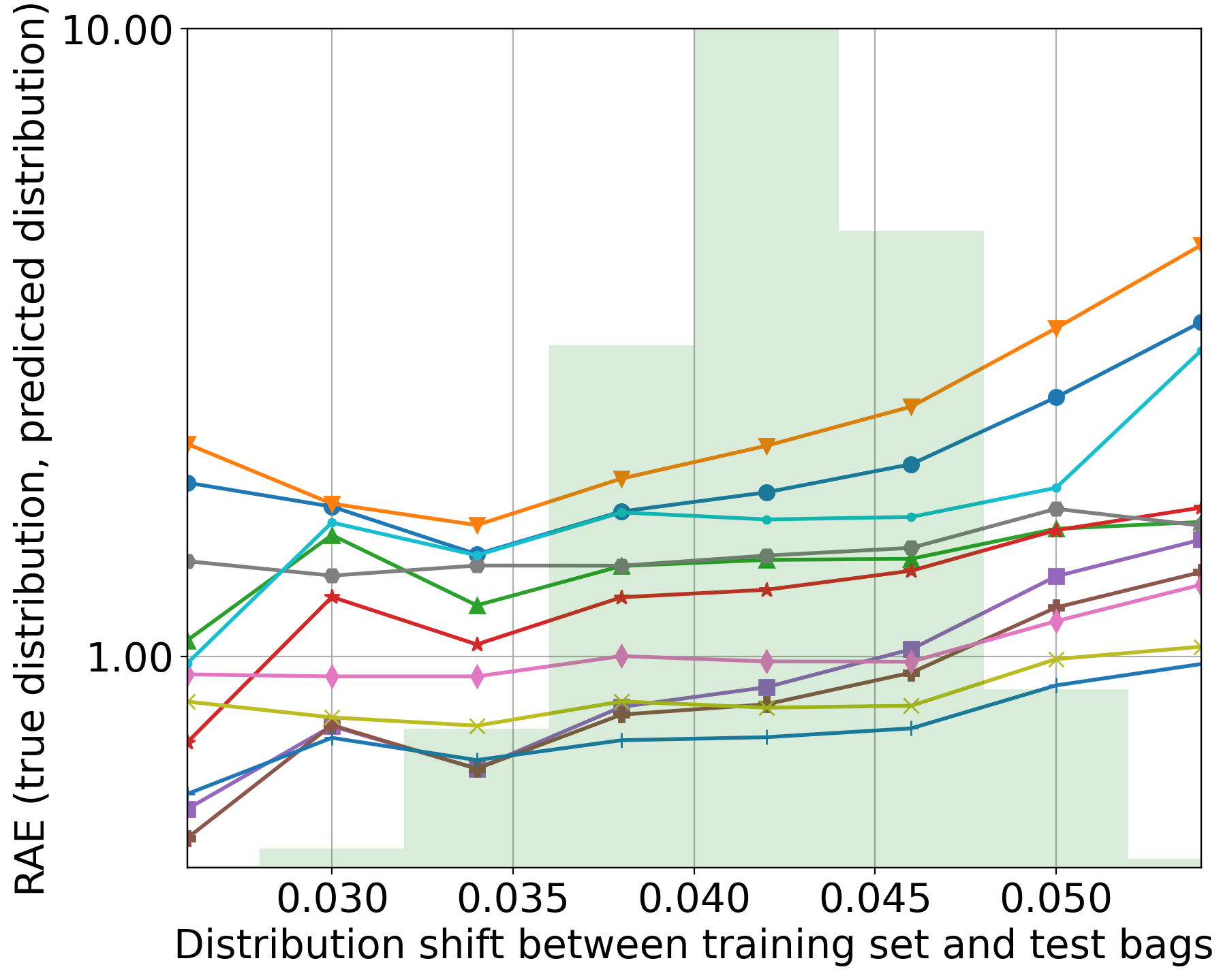}
  \caption{LeQua-T1B}
  \label{fig:t1b}
 \end{subfigure}
 \begin{subfigure}[T]{0.32\textwidth}
  \centering
  \includegraphics[trim = {0.9cm 0 0 0}, clip, width=5.4cm]{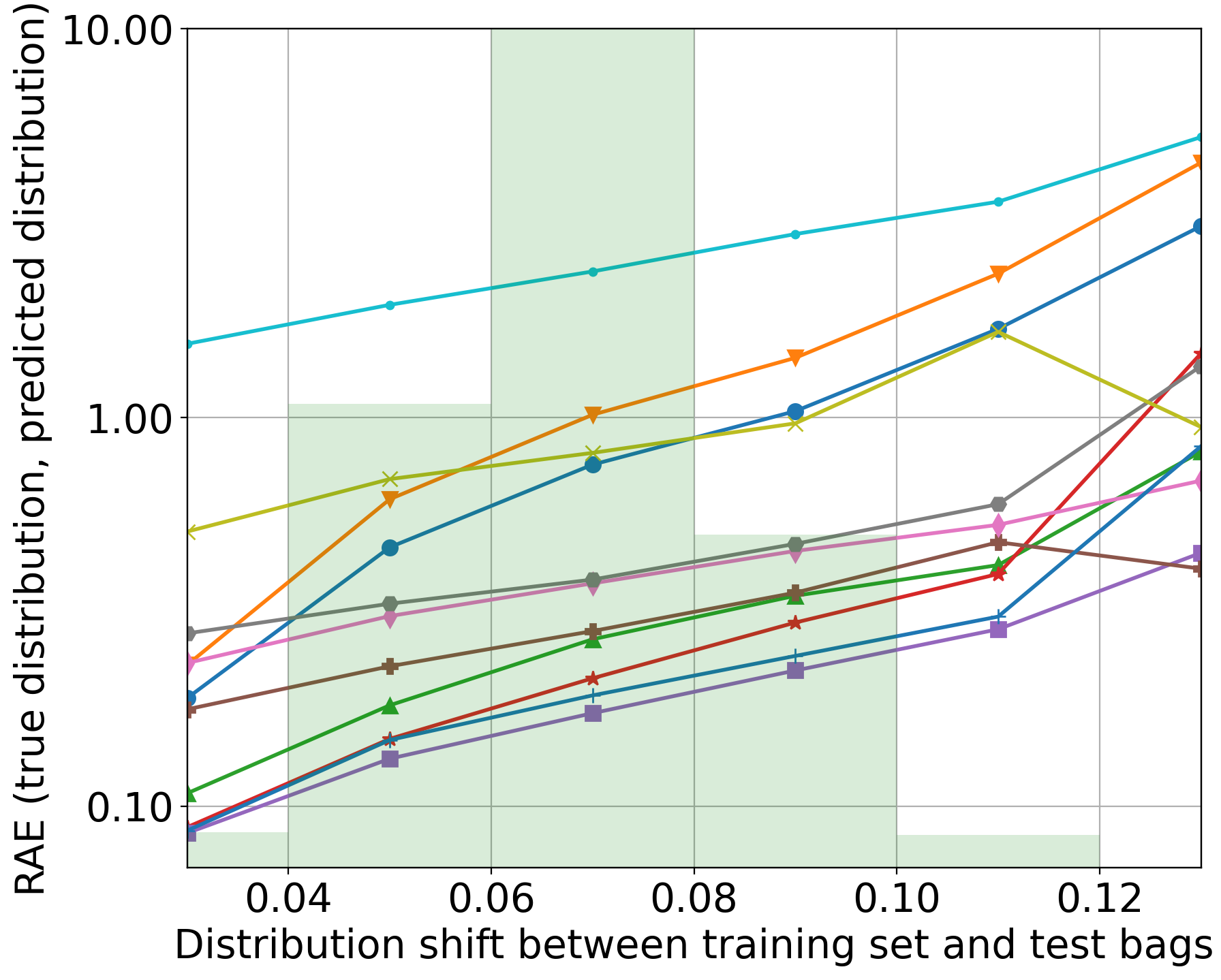}
  \caption{Fashion-MNIST}
  \label{fig:fashionmnist}
 \end{subfigure}
\caption{
Error distribution (measured in terms of RAE on a logarithmic scale) binned by the amount of prior probability shift ($|p_{D}-p_{B}|$) between the training set and each test bag. The green bars represent the distribution of bags per bin.}
\label{fig:t1errors}
\end{figure*}

\subsection{A Common Architecture}
\label{exp:architecture}

In order to guarantee a fair comparison, we used the exact same network architecture, depicted in Figure~\ref{fig:architecture}, for all methods, replacing only the permutation-invariant layer. The architecture is very similar to the ones previously proposed for set-based problems  \citep{lee2019set,Qi2021,zaheer2017deep}. The first part of the network is in charge of extracting features from the input examples. 
For the LeQua datasets, we used a series of fully-connected layers each followed by a LeakyReLU activation function and dropout, while for Fashion-MNIST we used a simple CNN with two convolutional layers and one fully-connected layer as output (Figure~\ref{fig:architecture}).
These vectors are given as input to a permutation-invariant layer that generates a single bag 
embedding.
The output of this layer is passed then through a feed-forward module followed by a softmax activation, which finally outputs 
a vector containing the estimated prevalence values $\hat{\mathbf{p}}$. 
In order to backpropagate the errors, at least one complete bag 
must be processed.

All DNN methods were trained following the same exact procedure: the training set $D'$ was split into an actual 
training set and a validation set used for monitoring the validation loss; we applied early stopping after 20 epochs without improvement in validation. 
Hyperparameters (see supplementary material) were tuned with the aid of \textsc{optuna} \citep{optuna_2019}.

\subsection{Results and discussion}

Table~\ref{tab:resultsall} and Figure~\ref{fig:t1errors} report the results of both experiments. In our opinion, the most important result is that HistNetQ outperforms EMQ in both LeQua competitions. This result is remarkable because EMQ is considered one of the best quantification methods in the literature \citep{alexandari2020maximum} and because it was the winner of both subtasks in the competition \citep{lequa2022}. The improvement obtained by HistNetQ is substantial in T1B (more than $13\%$). We think that this is due to two factors: (i) T1B is a multiclass problem in which it is difficult to accurately estimate the posterior probabilities (a crucial element for EMQ), and 
(ii) the performance of EMQ suffers as the shift between the training set and the test bags increases (see Figure~\ref{fig:t1b}). 

Regarding the comparison of DNN methods, the results show that representing bags using histograms (HistNetQ) brings about better quantification performance than when using SetTransformers or simple aggregation functions (DeepSets) \citep{Qi2021} across the three datasets, and the difference in performance seems to correlate with the complexity of the problem, \blue{with T1B standing out as the most challenging dataset among them}. We conjecture that this improvement comes from the fact that histogram-based representations are naturally geared toward ``counting'', and this turns beneficial for quantification. Interestingly, DeepSets(median) obtains the second-best RAE score in T1B. This may be surprising because it uses an apparently simplistic pooling layer. The performance of SetTransformers is erratic: it performs similarly, in a statistically significant sense, to HistNetQ in T1A, but it obtains the worst results from the deep learning lot in T1B. T1B is undeniably harder than T1A 
and SetTransformers' inducing points likely struggled to capture the interactions between all the classes. We were unable to make SetTranformers converge to better results in this case, despite trying many combinations of its hyperparameters (number of inducing points, number of heads, etc.). Although transformers are powerful tools in many contexts, they seem not to be the most adequate solution for quantification tasks where the order and the relation between examples in a bag are less important \blue{(this is in contrast to other types of data, such as in natural language processing, where transformers excel in learning from the order and relations between words)}.

Yet another aspect that proved essential for avoiding overfitting in all DNN methods is the Bag Mixer heuristic. We analyze this in more detail in Section~\ref{exp:ablation}.
Concerning HistNetQ, we also analyzed the extent to which the number of bins affects performance (see Table~\ref{tab:numberofbins}). 
We observe that in complex problems, like LeQua-T1B, the performance of HistNetQ tends to improve as the number of bins increases, leading to networks with a higher number of parameters. However, this is not \blue{necessarily} a rule of thumb, because in simpler problems, having too many bins might lead to overfitting. We would therefore recommend treating the number of bins just as any other hyperparameter to be tuned for each specific problem.



\begin{table*}[h]
 \caption{Results by number of bins in  LeQua-T1B}
 \label{tab:numberofbins}
 \vskip 0.15in
 \centering \input{tables/bins_comp_decimals.tex}
 \vskip -0.1in
\end{table*}


The results on Fashion-MNIST show that EMQ with calibration is the best approach, even while requiring less computational resources than DNN methods. According to the literature, these results were to be expected, but the performance of HistNetQ is rather similar and not significantly worse; it is even slightly better for AE. However, in this case, HistNetQ performs worse when the amount of shift is large (see Figure~\ref{fig:fashionmnist}). On the other hand, HistNetQ outperforms the rest of the quantification algorithms (only PACC gets close) as well as DNN methods also in this case. As witnessed in the first experiment, DeepSets using median or average polling layers prove more stable than SetTransformers, especially for RAE. These results seem to suggest that HistNetQ is competitive and should be considered even for problems in which only individual labeled examples 
are available. 

\subsection{Ablation Study}
\label{exp:ablation}

%
%

As recalled from Section~\ref{sec:sample-mixer}, the Bag Mixer is a data augmentation technique meant to enhance the training data of DNN symmetric quantifiers in order to avoid overfitting. 
In this section, we analyze the extent to which the Bag Mixer impacts the performance of each network. To do so, we carry out additional experiments in which the Bag Mixer is not used (that is, using only the training bags provided in $D'$), and we compare the results with the previously reported in Table~\ref{tab:resultsall}.

For this experiment, we have used the LeQua datasets because the training bags with their corresponding prevalence values are provided and thus limited. In contrast, in Fashion-MNIST, training bags are generated with the APP protocol using the labeled training dataset $D$ and therefore are \blue{practically} unlimited and different in every training iteration.

For LeQua-T1A, Figure~\ref{fig:overfitting} and Table~\ref{table:t1a_wo_sampmix} show that the networks exhibit lower training errors when operating with just 1000 training bags. However, this reduction in training error leads to a significant drawback, as these models tend to perform notably worse on the validation and holdout datasets due to overfitting.

\begin{table*}[ht]
\centering
\caption{LeQua-T1A results without Bag Mixer.
Relative error variation with respect to when using Bag Mixer (Table~\ref{tab:resultsall}) is shown in parenthesis.}
\label{table:t1a_wo_sampmix}
\begin{tabular}{l|ll}
\toprule
 & AE & RAE \\
\midrule
Deepsets (avg) & 0.0326 (+17.3\%) & 0.1469 (+15.8\%) \\
Deepsets (median) & 0.0416 (+42.5\%) & 0.1810 (+30.3\%) \\ 
Deepsets (max) & 0.0570 (+14.2\%) & 0.2287 (+4.8\%)\\

SetTransformers & 0.0368 (+63.6\%) & 0.1553 (+41.1\%) \\
HistNetQ (ours) & 0.0279 (+24.6\%) & 0.1265 (+18.1\%)\\
\bottomrule
\end{tabular}
\end{table*}

In the case of LeQua-T1B, Table~\ref{table:t1b_wo_sampmix}, the adverse effects of overfitting are amplified. This was to be expected, since the number of classes is much higher in this dataset (up to 28) while the number of training bags stays the same (i.e., 1000). In this scenario, the networks, especially those with more complex architectures containing a greater number of parameters, such as SetTranformers or HistNetQ, are more prone to overfit.

\begin{figure*}[t]
 \centering
 \begin{subfigure}[T]{0.48\textwidth}
  \centering
  \includegraphics[trim={3cm 3cm 3cm 3cm},clip,width=\textwidth]{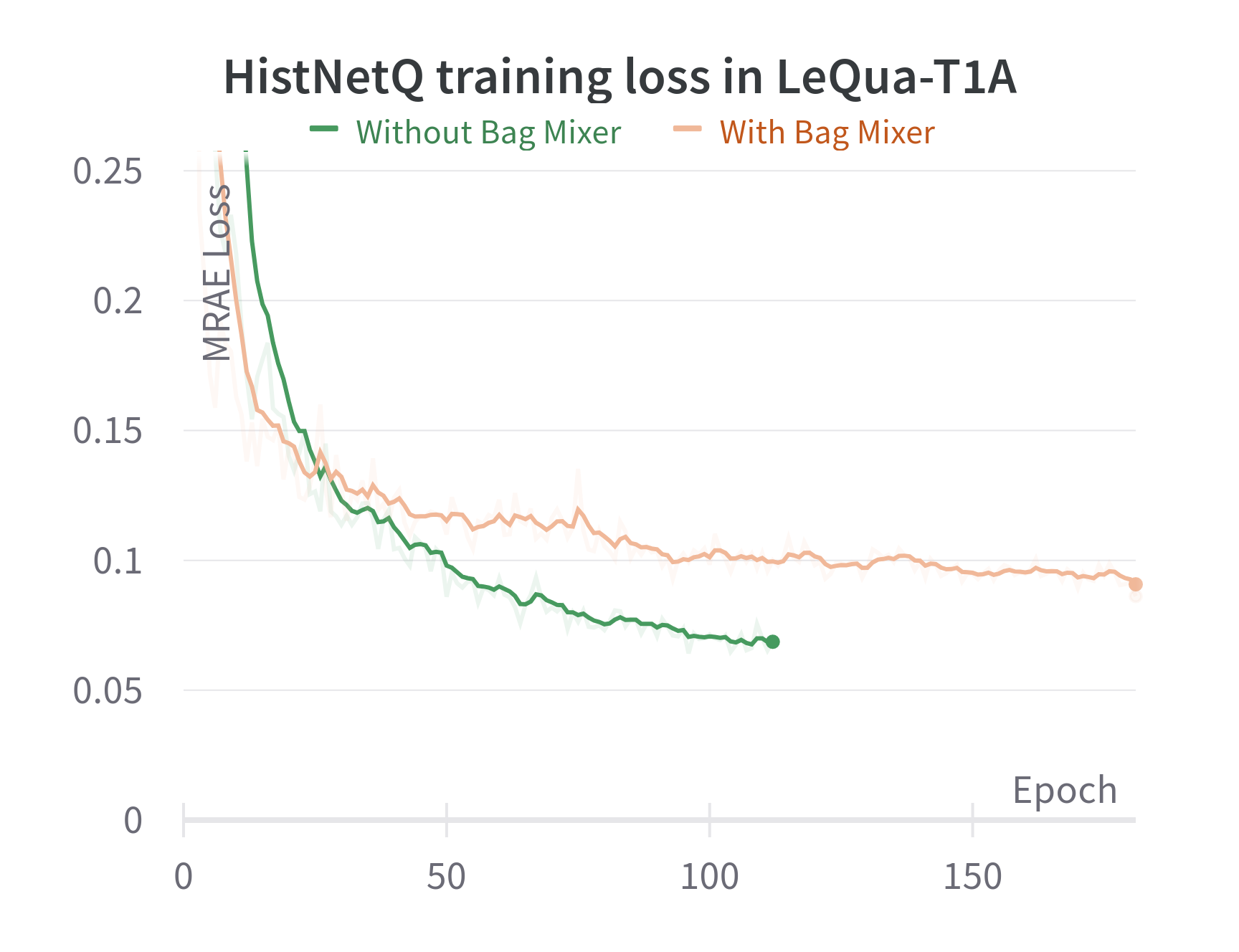}
  \label{fig:trainingloss}
 \end{subfigure}
 \hspace{0.02\textwidth}
 \begin{subfigure}[T]{0.48\textwidth}
  \centering
  \includegraphics[trim={3cm 3cm 3cm 3cm},clip,width=\textwidth]{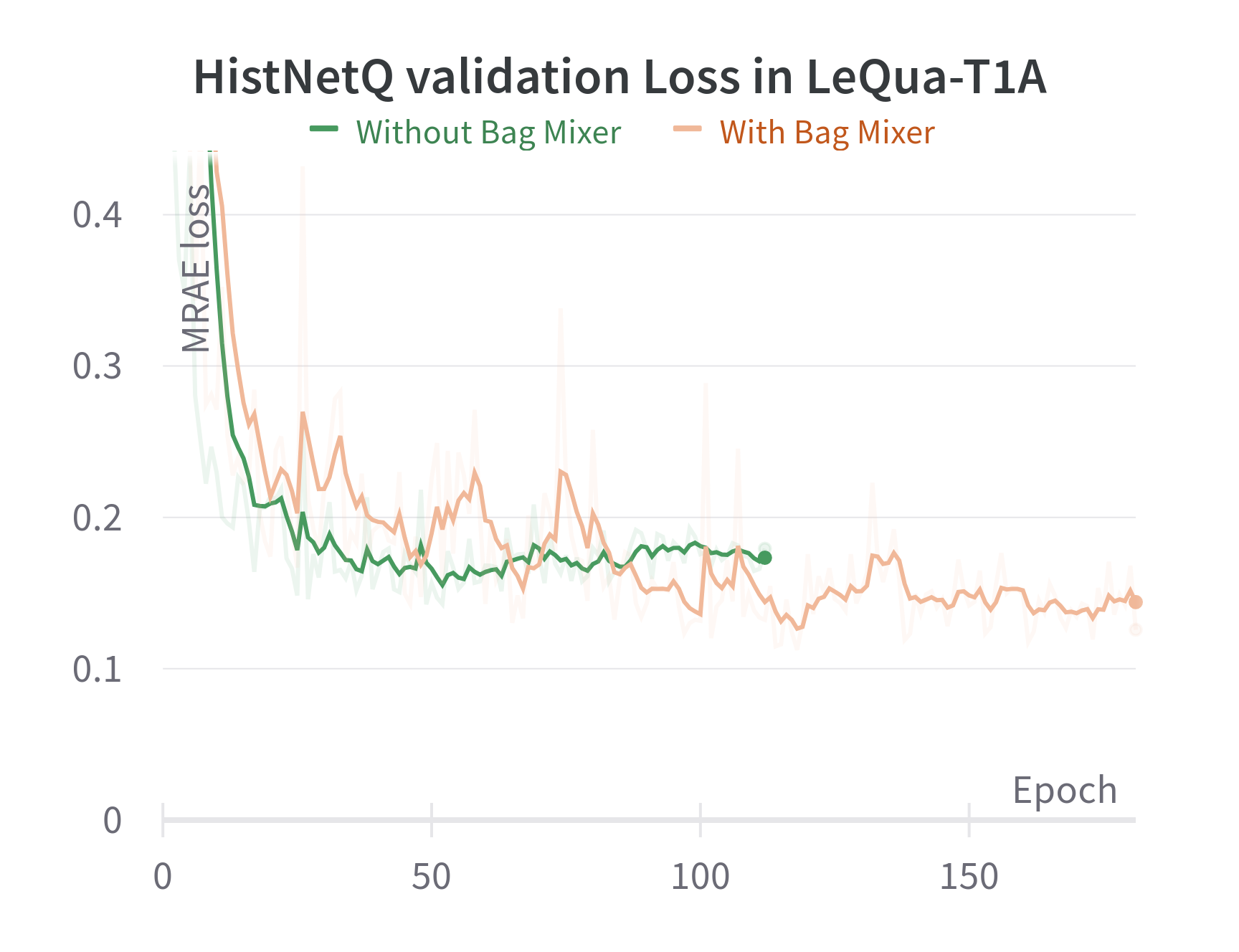}
  \label{fig:validationloss}
 \end{subfigure}
\caption{Training and validation loss trends of HistNetQ in LeQua-T1A, with and without the Bag Mixer using a patience criterion (i.e. stopping the training process after a number of consecutive epochs in which the validation loss does not improve).
The training loss decreases faster without the Bag Mixer (left figure); however, the validation loss keeps improving with the Bag Mixer (right figure).
This is an indication that the Bag Mixer helps counter overfitting. 
}
\label{fig:overfitting}
\end{figure*}

\begin{table*}[htbp]
\centering
\caption{LeQua-T1B results without Bag Mixer. 
Relative error variation with respect to when using the Bag Mixer (Table~\ref{tab:resultsall}) is shown in parenthesis.
}
\label{table:t1b_wo_sampmix}
\begin{tabular}{l|ll}
\toprule
 & AE & RAE \\
\midrule
Deepsets (avg) & 0.0449 (+250.8\%) & 1.5029 (+51\%) \\
Deepsets (median) & 0.0215 (+50.4\%) & 1.0991 (+30.2\%) \\
Deepsets (max) & 0.0200 (-27.8\%) & 1.5740 (+7.5\%) \\
SetTransformers & 0.0311 (-19.2\%) & 4.2416 (+153.3\%) \\
HistNetQ (ours) & 0.0445 (+315.9\%) & 1.5108 (+99.5\%) \\
\bottomrule
\end{tabular}
\end{table*}

\section{Conclusions}
\label{sec:conclusions}

This paper introduces HistNetQ, a DNN for quantification that relies on a permutation-invariant layer based on differentiable histograms.  We carried out experiments using two different quantification problems (from computer vision and text analysis) in which we compared the performance of HistNetQ against previously proposed networks for set processing and also against the most important algorithms from the quantification literature. The results show that HistNetQ achieved state-of-the-art performance in both problems. From a qualitative point of view, HistNetQ also displays interesting properties like i) the ability to directly learn from bags 
labeled by prevalence, which allows HistNetQ to be applied to scenarios in which traditional methods cannot; and ii) the possibility to directly optimize for specific loss functions. 

This research may hopefully offer a new viewpoint in quantification learning, since our results suggest that exploiting data labeled at the aggregate level might be preferable, in terms of quantification performance, than exploiting data labeled at the individual level. Overall, this study seems to suggest that HistNetQ is a promising alternative for implementing the symmetric approach in real applications, obtaining state-of-the art results that surpass previous approaches. 

Future work may include i) studying the capabilities of HistNetQ when confronted with types of dataset shift other than prior probability shift \citep{tasche2022class,zhang2013domain} and ii) exploring potential applications of this architecture to other problems that, like quantification, require learning a model from density estimates over sets of examples.

\bibliographystyle{unsrtnat}
\bibliography{histnet_bib} 
\clearpage
\appendix
\section{Other Types of Differentiable Histogram Layers}
\label{sec:app:architectures}

Differentiable histograms can be classified as belonging to the \emph{hard} or \emph{soft} binning types, depending on whether each value contributes only to the bin it belongs to or if instead, each value contributes to more than one bin (based on the distance of the value to the bin center and its width), respectively.

In our experiments, we have tested four different histograms proposed in the literature, all of them permutation-invariant and therefore suitable for set processing. In Section 5, we only reported results for the \emph{hard} variant that obtained slightly better results. 
The architecture of the remaining histogram types, along with the results we have obtained in our experiments, are discussed in this section.

\subsection{Architectures of Other Differentiable Histograms}
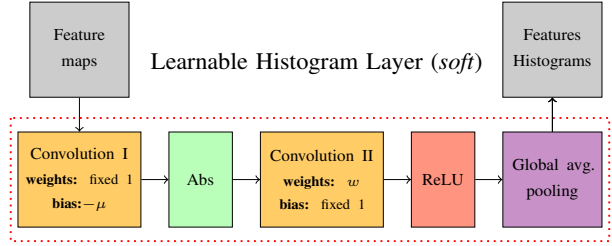
\begin{figure}[!ht]
\centering
\scalebox{0.9}{
\begin{tikzpicture}[
    node distance=5mm and 4 mm,
      start chain = going right,
every node/.style = {
    draw,
    minimum height=4em,
    text width=3.5em,
    align=center,
    join},
every join/.style = {->}
            ]
        \node[on chain,fill=\GrayColor] (featuremaps){\scriptsize Feature\\maps};
        \node[on chain,text width=4.5em]  (conv1)[fill=\ConvColor,below=of featuremaps]{\scriptsize Convolution I\\[0.5em] \tiny \textbf{weights:} fixed 1\\[0.5em]\textbf{bias:}$-\mu$\\};
        \node[on chain,fill=\AbsColor,text width=2em]  {\scriptsize Abs};
        \node[on chain,text width=4.5em]  (conv2)[fill=\ConvColor]{\scriptsize Convolution II\\[0.5em] \tiny \textbf{weights:} $w$\\[0.5em]\textbf{bias:} fixed 1\\};
        \node  [on chain,fill=\ReluColor,text width=2em]{\scriptsize ReLU};
        \node  [on chain,fill=\AvgColor](avg){\scriptsize Global avg. pooling};
        \node  [on chain,above=of avg,fill=\GrayColor]{\scriptsize Features Histograms};
        \draw[red,dotted,thick] ($(conv1.north west)+(-0.1,0.2)$)  rectangle ($(avg.south east)+(0.1,-0.2)$);
        \coordinate (CENTER) at ($(conv1)!0.5!(avg)$);
        \node  [draw=none,align=center,text width=20em,shift={(0,0.5)},above =of CENTER,]{Learnable Histogram Layer (\emph{soft})};
    \end{tikzpicture}}
    \caption{\emph{soft}: Learnable histogram layer with soft binning and variable bin centers and widths} \label{fig:histsoft}
\end{figure}
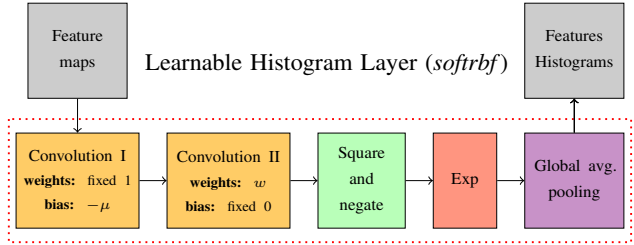
\begin{figure}[!ht]
\centering
\scalebox{0.9}{
\begin{tikzpicture}[
    node distance=5mm and 4 mm,
      start chain = going right,
every node/.style = {
    draw,
    minimum height=4em,
    text width=3.5em,
    align=center,
    join},
every join/.style = {->}
            ]
        \node[on chain,fill=\GrayColor] (featuremaps){\scriptsize Feature\\maps};
        \node[on chain,text width=4.5em]  (conv1)[fill=\ConvColor,below=of featuremaps]{\scriptsize Convolution I\\[0.5em] \tiny \textbf{weights:} fixed 1\\[0.5em]\textbf{bias: }$-\mu$\\};
        \node[on chain,text width=4.5em]  (conv2)[fill=\ConvColor]{\scriptsize Convolution II\\[0.5em] \tiny \textbf{weights:} $w$\\[0.5em]\textbf{bias:} fixed 0\\};
        \node[on chain,fill=\AbsColor,text width=3em](square) {\scriptsize Square and negate};
        \node  [on chain,fill=\ReluColor,text width=2em]{\scriptsize Exp};
        \node  [on chain,fill=\AvgColor](avg){\scriptsize Global avg. pooling};
        \node  [on chain,above=of avg,fill=\GrayColor]{\scriptsize Features Histograms};
        \draw[red,dotted,thick] ($(conv1.north west)+(-0.1,0.2)$)  rectangle ($(avg.south east)+(0.1,-0.2)$);
        \coordinate (CENTER) at ($(conv1)!0.5!(avg)$);
        \node  [draw=none,align=center,text width=20em,shift={(0,0.5)},above =of CENTER,]{Learnable Histogram Layer (\emph{softrbf})};
    \end{tikzpicture}}
\caption{\emph{softrbf}: Learnable histogram layer with soft binning and variable bin centers and widths using a RBF function}\label{fig:histsoftrbf}
\end{figure}

Given a bag of $n$ data examples $B=\{\mathbf{x}_i\}_{i=1}^{n}$, with $\mathbf{x}_i\in\inputspace$ 
as before, our goal is to compute a histogram for every feature vector $\{\mathbf{f}_i\}_{i=1}^z$, where $\mathbf{f}_k\in\mathbb{R}^n$ represents the values of the $k$-th feature across the instances $\mathbf{x}_i \in B$ in a latent space $\mathbb{R}^z$.

The \emph{soft} differentiable histogram layers, proposed by \cite{wang2016learnable}, use soft binning with variable bins and employ convolutional layers to approximate the histogram (see Figure~\ref{fig:histsoft}). 
The counts in the soft histograms are computed by:
\begin{equation}
    \phi(v; \mu, w)=\max{\left(0,1-\frac{1}{w}\times|v-\mu| \right)},
\end{equation}
where $\mu$ and $w$ are the bin center and width of the bin, and $v$ is one of the values generated for the $k$-th feature. The rationale behind this equation is that the closer the value $v$ gets to the bin center $\mu$, the smaller the multiplier becomes, thus returning a value close to 1. Analogously, a small value for the bin width $w$ results in larger multiplicative factors, thus making the final count closer to 0.

As shown in Figure~\ref{fig:histsoft}, a first convolution layer learns the bin centers through the bias term while a second convolutional layer learns the bin widths.

In the \emph{softrbf} differentiable histogram layers by \cite{peeples2021histogram}, an RBF function is used to approximate the histogram. Just like in the previous case, the histogram falls into the category of soft binning with variable bins. In this case, the counts are computed as:
\begin{equation}
    \phi(v; \mu, w)=e^{-\left(\frac{v-\mu}{w}\right)^2}.
\end{equation}
The parameters of this function are similarly learned through convolutions 
(see Figure~\ref{fig:histsoftrbf}).

Finally, the \emph{sigmoid} differentiable histogram layers do not use convolutional layers but two logistic functions to approximate each bin. In this case, the type of histogram produced is hard with fixed bin centers and widths (there are no learnable parameters in the layer). In the first step, the bin centers $\mu$ and the bin width $w$ are initialized, depending on the number of bins selected. Then, in a second step, two logistic functions are used to approximate which values fall in each bin (see Figure~\ref{fig:sigmoid}). This method can be easily computed using basic differentiable operations. 

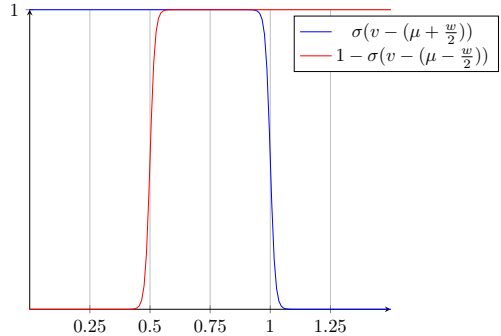
\begin{figure}[!ht]
\centering
\scalebox{0.7}{
\begin{tikzpicture}[declare function={
sigma1(\x)=1/(1+exp(100*(\x-(0.75+0.25))));
sigma2(\x)=1-(1/(1+exp(100*(\x-(0.75-0.25)))));}]
\begin{axis}%
[
    grid=major,     
    xmin=0,
    xmax=1.5,
    axis x line=bottom,
    ytick={0,1},
    xtick={0.25,0.5,0.75,1,1.25},
    ymax=1,
    axis y line=middle,
    samples=1000,
    domain=-5:5,
    legend style={at={(1.3,0.98)}}     
]
    \addplot[blue,mark=none]   (x,{sigma1(x)});
    \addplot[red,mark=none]   (x,{sigma2(x)});
    \legend{$\sigma(v-(\mu+\frac{w}{2}))$,$1-\sigma(v-(\mu-\frac{w}{2}))$}
\end{axis}
\end{tikzpicture}}
\caption{Histogram bin approximation using two sigmoid functions $\sigma(v)=\frac{1}{1+e^{\gamma v}}$. In this example, the bin center is fixed and equal to $\mu=0.75$. Bin width is also fixed with $w=0.5$. $\gamma$ is a constant with a high enough value to make sigmoid functions sharp and closer to a step function.}
\label{fig:sigmoid}
\end{figure}

\section{Additional Results}

In this section, we turn to report additional experiments we have carried out that were omitted from the paper for the sake of brevity.
In particular, we report results for the differentiable histograms (HistNetQ Soft, HistNetQ SoftRBF, and HistNetQ Sigmoid) discussed in the previous section.  The results are displayed in Table \ref{tab:resultsall_apendix}.



\begin{table*}[th]
 \caption{Results for \textsc{Fashion-MNIST}, \textsc{LeQua-T1A} and \textsc{LeQua-T1B} in terms of AE and RAE.}
 \label{tab:resultsall_apendix}
 \vskip 0.15in
 \resizebox{\textwidth}{!}{%
 \centering { 
 \begin{tabular}{r|cccccc}
\toprule
&\multicolumn{2}{c}{Fashion-MNIST}  & \multicolumn{2}{c}{LeQua-T1A} & \multicolumn{2}{c}{LeQua-T1B}  \\ 
&AE & RAE & AE  & RAE & AE & RAE           \\
\midrule
CC & 0.01634 $\pm$ 0.00738 & 0.58279 $\pm$ 0.72314 & 0.09160 $\pm$ 0.05540 & 1.08400 $\pm$ 4.31090& 0.01406 $\pm$ 0.00295 & 1.89365 $\pm$ 1.18732 \\
PCC & 0.02040 $\pm$ 0.00796 & 0.78168 $\pm$ 0.97434& 0.11664 $\pm$ 0.06978 & 1.39402 $\pm$ 5.62123 & 0.01711 $\pm$ 0.00332 & 2.26462 $\pm$ 1.41627 \\
ACC & 0.00824 $\pm$ 0.00310 & 0.22256 $\pm$ 0.23753& 0.03716 $\pm$ 0.02936 & 0.17020 $\pm$ 0.50800 & 0.01841 $\pm$ 0.00437 & 1.42134 $\pm$ 1.26971 \\
PACC & 0.00673 $\pm$ 0.00238 & 0.18310 $\pm$ 0.19252& 0.02985 $\pm$ 0.02258 & 0.15218 $\pm$ 0.46440 & 0.01578 $\pm$ 0.00379 & 1.30538 $\pm$ 0.98837 \\
HDy & - & - & 0.02814 $\pm$ 0.02212 & 0.14514 $\pm$ 0.45621 & - & -\\
QuaNet & - & - & 0.03418 $\pm$ 0.02528 & 0.31764 $\pm$ 1.35237 & - & - \\
EMQ-BCTS & 0.00652 $\pm$ 0.00246 & 0.15097 $\pm$ 0.15191 & 0.02689 $\pm$ 0.02094 & 0.11828 $\pm$ 0.25065  & 0.01174 $\pm$ 0.00305 & 0.93721 $\pm$ 0.81732  \\
EMQ-NoCalib & 0.01324 $\pm$ 0.00472 & 0.25493 $\pm$ 0.22246 & 0.02359 $\pm$ 0.01845 & 0.10878 $\pm$ 0.26668  & 0.01177 $\pm$ 0.00285 & 0.87802 $\pm$ 0.75120  \\
\hline
DeepSets (avg) & 0.00829 $\pm$ 0.00254 & 0.32826 $\pm$ 0.23251& 0.02779 $\pm$ 0.02105 & 0.12686 $\pm$ 0.22817 & 0.01283 $\pm$ 0.00379 & 0.99542 $\pm$ 0.65778 \\
DeepSets (median) & 0.00942 $\pm$ 0.00288 & 0.71946 $\pm$ 0.58579& 0.02919 $\pm$ 0.02273 & 0.13887 $\pm$ 0.25631 & 0.01429 $\pm$ 0.00432 & 0.84427 $\pm$ 0.54286 \\
DeepSets (max) & 0.02185 $\pm$ 0.00699 & 0.35195 $\pm$ 0.32316& 0.04991 $\pm$ 0.04167 & 0.21830 $\pm$ 0.48828 & 0.02766 $\pm$ 0.00515 & 1.46464 $\pm$ 1.02644 \\
SetTransformers & 0.01043 $\pm$ 0.00328  & 2.20175 $\pm$ 1.19007 & \ddag{0.02246 $\pm$ 0.01717} & \ddag{0.10958 $\pm$ 0.26205} & 0.03847 $\pm$ 0.00779 & 1.67475 $\pm$ 1.42750 \\
HistNetQ Hard (ours) & \textbf{0.00602 $\pm$ 0.00206} & 0.15923 $\pm$ 0.17085 & 0.02236 $\pm$ 0.01709 & \textbf{0.10707 $\pm$ 0.23312} & \textbf{0.01070 $\pm$ 0.00367} & \textbf{0.75739 $\pm$ 0.48891} \\
HistNetQ Soft (ours) & 0.00842 $\pm$ 0.00273 & 0.16616 $\pm$ 0.14238 & 0.02279 $\pm$ 0.01763 & \dag{0.10830 $\pm$ 0.22461} & 0.01846 $\pm$ 0.00671 & 0.94806 $\pm$ 0.58838 \\
HistNetQ SoftRBF (ours) & 0.00688 $\pm$ 0.00216 & \textbf{0.13601 $\pm$ 0.11244} & \dag{0.02257 $\pm$ 0.01729} & \ddag{0.11250 $\pm$ 0.28344} & 0.02095 $\pm$ 0.00688 & 1.05116 $\pm$ 0.66311\\
HistNetQ Sigmoid (ours) & 0.00758 $\pm$ 0.00237 & 0.69029 $\pm$ 0.50995 & \textbf{0.02197 $\pm$ 0.01746} & 0.10728 $\pm$ 0.27898 & 0.01855 $\pm$ 0.00630 & 0.99868 $\pm$ 0.62058 \\
\bottomrule
\end{tabular}}}
\end{table*}

These results show that the differences in performance between the histogram-based models are rather small for Fashion-MNIST (in which SoftRBF seems to work better in terms of RAE), and in LeQua-T1A (in a statistically significant sense). However, for LeQua-T1B (the hardest problem in terms of the number of classes), HistNetQ Hard clearly stands out as the best of the lot. It is worth noting that these results, obtained in the test sets, align well with the trends each method displays in the validation loss. Indeed, our preference for HistNetQ Hard over the rest of the methods is based on the observation that HistNetQ Hard displays the smallest validation loss overall --that is to say, we have not simply picked the method displaying the best results in test.

\newpage
\section{Hyper-parameter Selection}


\begin{table*}[ht]
\caption{Summary of the most important hyperparameters used for each task. The last four hyperparameters are specific for SetTransformers}
\vskip 0.15in
\centering
\resizebox{\textwidth}{!}{%
\begin{tabular}{rp{6cm}ccc}
\toprule
Hyperparam. & \centering Description & Fashion-MNIST & LeQua-T1A & LeQua-T1B \\
\midrule
$lr$ & Starting learning rate & 0.0003 & 0.0001 & 0.0005 \\
\emph{optimizer} & Optimizer used for training & AdamW & AdamW & AdamW \\
\emph{batch\_size} & Number of full bags passed through the network before updating the weights & 2 & 20 & 500 \\
$wd$ & Weight decay & 0 & 0.00001 & 0.00001 \\
$N$ & Number of bins used in the histograms & 32 & 32 & 64 \\
$R$ & Real bags proportion used by the Bag Mixer at each epoch & - & 0.9 & 0.5 \\
$z$ & Output size of the feature extraction layer & 256 & 300 & 512 \\
$FF\_Q$ & Number and size of linear layers in the quantification head & [1024] & [2048, 2048, 2048] & [4096] \\
\emph{dropout} & Dropout used in quantification module linear layers & 0.1 & 0.5 & 0.5 \\
\hline
$O$ & Output size for SetTransfomer & 512 & 512 & 512 \\
$I$ & Number of inducing points in SetTransfomer & 32 & 32 & 128 \\
$H$ & Hidden size in SetTransfomer & 256 & 256 & 256 \\
$nh$ & Number of heads in SetTransfomer & 4 & 4 & 4 \\
\bottomrule
\end{tabular}
}%
\label{tab:hyperparameters}
\end{table*}

Hyper-parameters for the different tasks were optimized with the help of \textsc{optuna} \citep{optuna_2019} for the cases in which it was computationally feasible (Fashion-MNIST and LeQua-T1A). Table~\ref{tab:hyperparameters} summarizes the configuration shown for the results presented in our paper.







\end{document}

%% file: tables/table_all_decimals.tex
\begin{tabular}{r|cccccc}
\toprule
& \multicolumn{2}{c}{LeQua-T1A} & \multicolumn{2}{c}{LeQua-T1B}  & \multicolumn{2}{c}{Fashion-MNIST} \\ 
&AE & RAE & AE  & RAE & AE & RAE           \\
\midrule
CC  & 0.0916 $\pm$ 0.055 & 1.0840 $\pm$ 4.311& 0.0141 $\pm$ 0.003 & 1.8936 $\pm$ 1.187  & 0.0163 $\pm$ 0.007 & 0.5828 $\pm$ 0.723\\
PCC & 0.1166 $\pm$ 0.070 & 1.3940 $\pm$ 5.621 & 0.0171 $\pm$ 0.003 & 2.2646 $\pm$ 1.416 & 0.0204 $\pm$ 0.008 & 0.7817 $\pm$ 0.974\\
ACC  &0.0372 $\pm$ 0.029 & 0.1702 $\pm$ 0.508 & 0.0184 $\pm$ 0.004 & 1.4213 $\pm$ 1.270 & 0.0082 $\pm$ 0.003 & 0.2226 $\pm$ 0.238\\
PACC & 0.0298 $\pm$ 0.023 & 0.1522 $\pm$ 0.464 & 0.0158 $\pm$ 0.004 & 1.3054 $\pm$ 0.988 & 0.0067 $\pm$ 0.002 & 0.1831 $\pm$ 0.193\\
HDy  & 0.0281 $\pm$ 0.022 & 0.1451 $\pm$ 0.456 & - & - & - & -\\
QuaNet & 0.0342 $\pm$ 0.025 & 0.3176 $\pm$ 1.352 & - & - & - & - \\
EMQ-BCTS & 0.0269 $\pm$ 0.021 & 0.1183 $\pm$ 0.251 &0.0117 $\pm$ 0.003 & 0.9372 $\pm$ 0.817 & 0.0065 $\pm$ 0.002 & \textbf{0.1510 $\pm$ 0.152} \\
EMQ-NoCalib  & 0.0236 $\pm$ 0.018 & 0.1088 $\pm$ 0.267  & 0.0118 $\pm$ 0.003 & 0.8780 $\pm$ 0.751  & 0.0132 $\pm$ 0.005 & 0.2549 $\pm$ 0.222 \\
\hline
DeepSets (avg) & 0.0278 $\pm$ 0.021 & 0.1269 $\pm$ 0.228 & 0.0128 $\pm$ 0.004 & 0.9954 $\pm$ 0.658 & 0.0083 $\pm$ 0.003 & 0.3283 $\pm$ 0.233\\
DeepSets (med) & 0.0292 $\pm$ 0.023 & 0.1389 $\pm$ 0.256 & 0.0143 $\pm$ 0.004 & 0.8443 $\pm$ 0.543 & 0.0094 $\pm$ 0.003 & 0.7195 $\pm$ 0.586 \\
DeepSets (max) & 0.0499 $\pm$ 0.042 & 0.2183 $\pm$ 0.488 & 0.0277 $\pm$ 0.005 & 1.4646 $\pm$ 1.026 & 0.0219 $\pm$ 0.007 & 0.3520 $\pm$ 0.323\\
SetTransformers & \ddag{0.0225 $\pm$ 0.017} & \ddag{0.1096 $\pm$ 0.262} & 0.0385 $\pm$ 0.008 & 1.6748 $\pm$ 1.428 & 0.0104 $\pm$ 0.003  & 2.2017 $\pm$ 1.190 \\
HistNetQ (ours) & \textbf{0.0224 $\pm$ 0.017} & \textbf{0.1071 $\pm$ 0.233} & \textbf{0.0107 $\pm$ 0.004} & \textbf{0.7574 $\pm$ 0.489} & \textbf{0.0060 $\pm$ 0.002} & \ddag{0.1592 $\pm$ 0.171}\\
\bottomrule
\end{tabular}

%% file: tables/bins_comp_decimals.tex
\begin{tabular}{r|rr}
\toprule
{} & {AE} & {RAE} \\
\midrule
HistNetQ ( 8 bins) & 0.0297 $\pm$ 0.008 & 1.2878 $\pm$ 1.000 \\
HistNetQ (16 bins) & 0.0212 $\pm$ 0.007 & 1.0572 $\pm$ 0.738 \\
HistNetQ (32 bins) & 0.0121 $\pm$ 0.005 & 0.7851 $\pm$ 0.520 \\
HistNetQ (64 bins) & \textbf{0.0107 $\pm$ 0.004} & \textbf{0.7574 $\pm$ 0.489} \\
\bottomrule
\end{tabular}

%% file: template.bbl
\begin{thebibliography}{43}
\providecommand{\natexlab}[1]{#1}
\providecommand{\url}[1]{\texttt{#1}}
\expandafter\ifx\csname urlstyle\endcsname\relax
  \providecommand{\doi}[1]{doi: #1}\else
  \providecommand{\doi}{doi: \begingroup \urlstyle{rm}\Url}\fi

\bibitem[Beijbom et~al.(2015)Beijbom, Hoffman, Yao, Darrell, Rodriguez-Ramirez, Gonzalez-Rivero, and Guldberg]{Beijbom2015}
Oscar Beijbom, Judy Hoffman, Evan Yao, Trevor Darrell, Alberto Rodriguez-Ramirez, Manuel Gonzalez-Rivero, and Ove~Hoegh Guldberg.
\newblock Quantification in-the-wild: data-sets and baselines.
\newblock \emph{arXiv:1510.04811 [cs]}, November 2015.
\newblock arXiv: 1510.04811.

\bibitem[Forman(2006)]{Forman2006}
George Forman.
\newblock Quantifying trends accurately despite classifier error and class imbalance.
\newblock In \emph{Proceedings of the 12th ACM SIGKDD International Conference on Knowledge Discovery and Data Mining (KDD 2006)}, pages 157--166, Philadelphia, {US}, 2006.
\newblock \doi{10.1145/1150402.1150423}.

\bibitem[Gonz{\'a}lez et~al.(2019)Gonz{\'a}lez, Casta{\~n}o, Peacock, D{\'\i}ez, Del~Coz, and Sosik]{gonzalez2019automatic}
Pablo Gonz{\'a}lez, Alberto Casta{\~n}o, Emily~E Peacock, Jorge D{\'\i}ez, Juan~Jos{\'e} Del~Coz, and Heidi~M Sosik.
\newblock Automatic plankton quantification using deep features.
\newblock \emph{Journal of Plankton Research}, 41\penalty0 (4):\penalty0 449--463, 2019.

\bibitem[Hopkins and King(2010)]{Hopkins2010}
Daniel Hopkins and Gary King.
\newblock A {method} of {automated} {nonparametric} {content} {analysis} for {social} {science}.
\newblock \emph{American Journal of Political Science}, 54\penalty0 (1):\penalty0 229--247, 2010.

\bibitem[Moreo and Sebastiani(2022)]{Moreo:2022bf}
Alejandro Moreo and Fabrizio Sebastiani.
\newblock Tweet sentiment quantification: An experimental re-evaluation.
\newblock \emph{PLOS ONE}, 17\penalty0 (9):\penalty0 1--23, 09 2022.
\newblock \doi{10.1371/journal.pone.0263449}.

\bibitem[Dias et~al.(2022)Dias, Ponti, and Minghim]{dias2022classification}
Fabio~Felix Dias, Moacir~Antonelli Ponti, and Rosane Minghim.
\newblock A classification and quantification approach to generate features in soundscape ecology using neural networks.
\newblock \emph{Neural Computing and Applications}, 34\penalty0 (3):\penalty0 1923--1937, 2022.

\bibitem[González et~al.(2017)González, Díez, Chawla, and del Coz]{Gonzalez2017}
Pablo González, Jorge Díez, Nitesh Chawla, and Juan~José del Coz.
\newblock Why is quantification an interesting learning problem?
\newblock \emph{Progress in Artificial Intelligence}, 6\penalty0 (1):\penalty0 53--58, 2017.
\newblock ISSN 2192-6360.
\newblock \doi{10.1007/s13748-016-0103-3}.

\bibitem[Forman(2008)]{Forman2008}
George Forman.
\newblock Quantifying counts and costs via classification.
\newblock \emph{Data Mining and Knowledge Discovery}, 17\penalty0 (2):\penalty0 164--206, October 2008.
\newblock ISSN 1573-756X.
\newblock \doi{10.1007/s10618-008-0097-y}.

\bibitem[González-Castro et~al.(2013)González-Castro, Alaiz-Rodríguez, and Alegre]{GonzalezCastro2013}
Víctor González-Castro, Rocío Alaiz-Rodríguez, and Enrique Alegre.
\newblock Class distribution estimation based on the {Hellinger} distance.
\newblock \emph{Information Sciences}, 218:\penalty0 146--164, January 2013.
\newblock ISSN 0020-0255.
\newblock \doi{10.1016/j.ins.2012.05.028}.

\bibitem[Kawakubo et~al.(2016)Kawakubo, Du~Plessis, and Sugiyama]{kawakubo2016computationally}
Hideko Kawakubo, Marthinus~Christoffel Du~Plessis, and Masashi Sugiyama.
\newblock Computationally efficient class-prior estimation under class balance change using energy distance.
\newblock \emph{IEICE TRANSACTIONS on Information and Systems}, 99\penalty0 (1):\penalty0 176--186, 2016.

\bibitem[Quionero-Candela et~al.(2009)Quionero-Candela, Sugiyama, Schwaighofer, and Lawrence]{quionero2009dataset}
Joaquin Quionero-Candela, Masashi Sugiyama, Anton Schwaighofer, and Neil~D. Lawrence.
\newblock \emph{Dataset Shift in Machine Learning}.
\newblock The MIT Press, Cambridge, MA, 2009.

\bibitem[Gonz{\'a}lez et~al.(2017)Gonz{\'a}lez, Casta{\~n}o, Chawla, and Coz]{gonzalez2017review}
Pablo Gonz{\'a}lez, Alberto Casta{\~n}o, Nitesh~V Chawla, and Juan Jos{\'e}~Del Coz.
\newblock A review on quantification learning.
\newblock \emph{ACM Computing Surveys (CSUR)}, 50\penalty0 (5):\penalty0 1--40, 2017.

\bibitem[Esuli et~al.(2023)Esuli, Fabris, Moreo, and Sebastiani]{quantbook2023}
Andrea Esuli, Alessandro Fabris, Alejandro Moreo, and Fabrizio Sebastiani.
\newblock \emph{Learning to Quantify}.
\newblock Springer, Cham, CH, 2023.
\newblock ISBN 978-3-031-20466-1.
\newblock \doi{10.1007/978-3-031-20467-8}.

\bibitem[Qi et~al.(2021)Qi, Khaleel, Tavanapong, Sukul, and Peterson]{Qi2021}
Lei Qi, Mohammed Khaleel, Wallapak Tavanapong, Adisak Sukul, and David Peterson.
\newblock A framework for deep quantification learning.
\newblock In \emph{Machine Learning and Knowledge Discovery in Databases: European Conference, ECML PKDD 2020, Ghent, Belgium, September 14--18, 2020, Proceedings, Part I}, pages 232--248. Springer, 2021.

\bibitem[Edwards and Storkey(2017)]{edwards2016towards}
Harrison Edwards and Amos~J. Storkey.
\newblock Towards a neural statistician.
\newblock In \emph{5th International Conference on Learning Representations, {ICLR} 2017, Toulon, France, April 24-26, 2017, Conference Track Proceedings}, 2017.

\bibitem[Murphy et~al.(2019)Murphy, Srinivasan, Rao, and Ribeiro]{janossypoolICLR2019}
Ryan~L. Murphy, Balasubramaniam Srinivasan, Vinayak~A. Rao, and Bruno Ribeiro.
\newblock Janossy pooling: Learning deep permutation-invariant functions for variable-size inputs.
\newblock In \emph{7th International Conference on Learning Representations, {ICLR} 2019, May 6-9, 2019}, New Orleans, LA, USA, 2019. OpenReview.net.

\bibitem[Wagstaff et~al.(2019)Wagstaff, Fuchs, Engelcke, Posner, and Osborne]{wagstaff2019limitations}
Edward Wagstaff, Fabian Fuchs, Martin Engelcke, Ingmar Posner, and Michael~A Osborne.
\newblock On the limitations of representing functions on sets.
\newblock In \emph{International Conference on Machine Learning}, pages 6487--6494. PMLR, 2019.

\bibitem[Zaheer et~al.(2017)Zaheer, Kottur, Ravanbakhsh, Poczos, Salakhutdinov, and Smola]{zaheer2017deep}
Manzil Zaheer, Satwik Kottur, Siamak Ravanbakhsh, Barnabas Poczos, Russ~R Salakhutdinov, and Alexander~J Smola.
\newblock Deep sets.
\newblock \emph{Advances in neural information processing systems}, 30, 2017.

\bibitem[Lee et~al.(2019)Lee, Lee, Kim, Kosiorek, Choi, and Teh]{lee2019set}
Juho Lee, Yoonho Lee, Jungtaek Kim, Adam Kosiorek, Seungjin Choi, and Yee~Whye Teh.
\newblock Set transformer: A framework for attention-based permutation-invariant neural networks.
\newblock In \emph{International conference on machine learning}, pages 3744--3753. PMLR, 2019.

\bibitem[Esuli et~al.(2018)Esuli, Moreo, and Sebastiani]{Esuli2018}
Andrea Esuli, Alejandro Moreo, and Fabrizio Sebastiani.
\newblock A recurrent neural network for sentiment quantification.
\newblock In \emph{Proceedings of the 27th ACM International Conference on Information and Knowledge Management (CIKM 2018)}, pages 1775--1778, Torino, {IT}, 2018.
\newblock \doi{10.1145/3269206.3269287}.

\bibitem[Esuli et~al.(2022)Esuli, Moreo, Sebastiani, and Sperduti]{lequa2022}
Andrea Esuli, Alejandro Moreo, Fabrizio Sebastiani, and Gianluca Sperduti.
\newblock A detailed overview of {L}e{Q}ua@{CLEF} 2022: Learning to quantify.
\newblock In \emph{Proceedings of the Working Notes of {CLEF} 2022 - Conference and Labs of the Evaluation Forum, Bologna, Italy, September 5th-8th, 2022}, volume 3180 of \emph{{CEUR} Workshop Proceedings}, pages 1849--1868, Bologna, Italy, 2022. CEUR-WS.org.

\bibitem[Fernandes~Vaz et~al.(2019)Fernandes~Vaz, Izbicki, and Bassi~Stern]{Vaz:2019eu}
Afonso Fernandes~Vaz, Rafael Izbicki, and Rafael Bassi~Stern.
\newblock Quantification under prior probability shift: {T}he ratio estimator and its extensions.
\newblock \emph{Journal of Machine Learning Research}, 20:\penalty0 79:1--79:33, 2019.

\bibitem[Lipton et~al.(2018)Lipton, Wang, and Smola]{lipton2018detecting}
Zachary Lipton, Yu-Xiang Wang, and Alexander Smola.
\newblock Detecting and correcting for label shift with black box predictors.
\newblock In \emph{International conference on machine learning}, pages 3122--3130. PMLR, 2018.

\bibitem[Bunse(2022)]{Bunse:2022oj}
Mirko Bunse.
\newblock On multi-class extensions of adjusted classify and count.
\newblock In \emph{Proceedings of the 2nd International Workshop on Learning to Quantify (LQ 2022)}, pages 43--50, Grenoble, IT, 2022.

\bibitem[Bella et~al.(2010)Bella, Ferri, Hernández-Orallo, and Ramírez-Quintana]{Bella2010}
Antonio Bella, Cesar Ferri, José Hernández-Orallo, and María~José Ramírez-Quintana.
\newblock Quantification via {Probability} {Estimators}.
\newblock In \emph{2010 {IEEE} {International} {Conference} on {Data} {Mining}}, pages 737--742, December 2010.
\newblock \doi{10.1109/ICDM.2010.75}.
\newblock ISSN: 2374-8486.

\bibitem[Saerens et~al.(2002)Saerens, Latinne, and Decaestecker]{Saerens2002}
Marco Saerens, Patrice Latinne, and Christine Decaestecker.
\newblock Adjusting the {outputs} of a {classifier} to {new} a {priori} {probabilities}: {A} {simple} {procedure}.
\newblock \emph{Neural Computation}, 14\penalty0 (1):\penalty0 21--41, January 2002.
\newblock ISSN 0899-7667.
\newblock \doi{10.1162/089976602753284446}.

\bibitem[Alexandari et~al.(2020)Alexandari, Kundaje, and Shrikumar]{alexandari2020maximum}
Amr Alexandari, Anshul Kundaje, and Avanti Shrikumar.
\newblock Maximum likelihood with bias-corrected calibration is hard-to-beat at label shift adaptation.
\newblock In \emph{International Conference on Machine Learning}, pages 222--232. PMLR, 2020.

\bibitem[Esuli et~al.(2020)Esuli, Molinari, and Sebastiani]{esuli2020critical}
Andrea Esuli, Alessio Molinari, and Fabrizio Sebastiani.
\newblock A critical reassessment of the {S}aerens-{L}atinne-{D}ecaestecker algorithm for posterior probability adjustment.
\newblock \emph{ACM Transactions on Information Systems (TOIS)}, 39\penalty0 (2):\penalty0 1--34, 2020.

\bibitem[Sebastiani(2020)]{Sebastiani2020}
Fabrizio Sebastiani.
\newblock Evaluation measures for quantification: {A}n axiomatic approach.
\newblock \emph{Information Retrieval Journal}, 23\penalty0 (3):\penalty0 255--288, 2020.
\newblock \doi{10.1007/s10791-019-09363-y}.

\bibitem[Foulds and Frank(2010)]{multiinstancereview2010}
James~R. Foulds and Eibe Frank.
\newblock A review of multi-instance learning assumptions.
\newblock \emph{Knowl. Eng. Rev.}, 25\penalty0 (1):\penalty0 1--25, 2010.
\newblock \doi{10.1017/S026988890999035X}.

\bibitem[de~Freitas and K{\"{u}}ck(2005)]{Freitas:2005qf}
Nando de~Freitas and Hendrik K{\"{u}}ck.
\newblock Learning about individuals from group statistics.
\newblock In \emph{Proceedings of the 21st Conference in Uncertainty in Artificial Intelligence (UAI 2005)}, pages 332--339, Edimburgh, UK, 2005.

\bibitem[Quadrianto et~al.(2009)Quadrianto, Smola, Caetano, and Le]{Quadrianto:2009lc}
Novi Quadrianto, Alexander~J. Smola, Tib{\'{e}}rio~S. Caetano, and Quoc~V. Le.
\newblock Estimating labels from label proportions.
\newblock \emph{Journal of Machine Learning Research}, 10:\penalty0 2349--2374, 2009.

\bibitem[Forman(2005)]{Forman2005}
George Forman.
\newblock Counting positives accurately despite inaccurate classification.
\newblock In \emph{Proceedings of the 16th European Conference on Machine Learning (ECML 2005)}, pages 564--575, Porto, {PT}, 2005.
\newblock \doi{10.1007/11564096_55}.

\bibitem[Smith and Tromble(2004)]{smith2004sampling}
Noah~A Smith and Roy~W Tromble.
\newblock Sampling uniformly from the unit simplex.
\newblock \emph{Johns Hopkins University, Tech. Rep}, 29, 2004.

\bibitem[Zhang et~al.(2013)Zhang, Sch{\"o}lkopf, Muandet, and Wang]{zhang2013domain}
Kun Zhang, Bernhard Sch{\"o}lkopf, Krikamol Muandet, and Zhikun Wang.
\newblock Domain adaptation under target and conditional shift.
\newblock In \emph{ICML}, pages 819--827, 2013.

\bibitem[Avi-Aharon et~al.(2020)Avi-Aharon, Arbelle, and Raviv]{avi2020deephist}
Mor Avi-Aharon, Assaf Arbelle, and Tammy~Riklin Raviv.
\newblock Deephist: Differentiable joint and color histogram layers for image-to-image translation.
\newblock \emph{arXiv preprint arXiv:2005.03995}, 2020.

\bibitem[Peeples et~al.(2022)Peeples, Xu, and Zare]{peeples2021histogram}
Joshua Peeples, Weihuang Xu, and Alina Zare.
\newblock Histogram layers for texture analysis.
\newblock \emph{IEEE Transactions on Artificial Intelligence}, 3\penalty0 (4):\penalty0 541--552, 2022.
\newblock \doi{10.1109/TAI.2021.3135804}.

\bibitem[Wang et~al.(2016)Wang, Li, Ouyang, and Wang]{wang2016learnable}
Zhe Wang, Hongsheng Li, Wanli Ouyang, and Xiaogang Wang.
\newblock Learnable histogram: Statistical context features for deep neural networks.
\newblock In \emph{European Conference on Computer Vision}, pages 246--262. Springer, 2016.

\bibitem[Yusuf et~al.(2020)Yusuf, Igwegbe, and Azeez]{yusuf2020differentiable}
Ibrahim Yusuf, George Igwegbe, and Oluwafemi Azeez.
\newblock Differentiable histogram with hard-binning.
\newblock \emph{arXiv preprint arXiv:2012.06311}, 2020.

\bibitem[Moreo et~al.(2021)Moreo, Esuli, and Sebastiani]{moreo2021quapy}
Alejandro Moreo, Andrea Esuli, and Fabrizio Sebastiani.
\newblock Qua{P}y: a python-based framework for quantification.
\newblock In \emph{Proceedings of the 30th ACM International Conference on Information \& Knowledge Management}, pages 4534--4543, 2021.

\bibitem[Xiao et~al.(2017)Xiao, Rasul, and Vollgraf]{xiao2017fashionmnist}
Han Xiao, Kashif Rasul, and Roland Vollgraf.
\newblock Fashion-mnist: a novel image dataset for benchmarking machine learning algorithms.
\newblock \emph{arXiv preprint arXiv:1708.07747}, 2017.

\bibitem[Akiba et~al.(2019)Akiba, Sano, Yanase, Ohta, and Koyama]{optuna_2019}
Takuya Akiba, Shotaro Sano, Toshihiko Yanase, Takeru Ohta, and Masanori Koyama.
\newblock Optuna: A next-generation hyperparameter optimization framework.
\newblock In \emph{Proceedings of the 25rd {ACM} {SIGKDD} International Conference on Knowledge Discovery and Data Mining}, 2019.

\bibitem[Tasche(2022)]{tasche2022class}
Dirk Tasche.
\newblock Class prior estimation under covariate shift: {N}o problem?
\newblock In \emph{Proceedings of the 2nd International Workshop on Learning to Quantify: Methods and Applications (LQ 2022), ECML/PKDD}, Grenoble (France), 2022.

\end{thebibliography}
